%% file: main_arxiv.tex
\documentclass[10pt]{article} %

\usepackage[preprint]{arxiv}

\usepackage{microtype}
\usepackage{graphicx}
\usepackage{subfigure}
\usepackage{booktabs} %

\usepackage[usenames,dvipsnames]{xcolor}
\usepackage[colorlinks=true,citecolor=ForestGreen,linkcolor=Bittersweet]{hyperref}

\usepackage{amsmath}
\usepackage{amssymb}
\usepackage{mathtools}
\usepackage{amsthm}
\usepackage{dsfont}
\usepackage{pifont}
\usepackage[capitalize,noabbrev]{cleveref}
\usepackage[normalem]{ulem}

\usepackage[toc,page,header]{appendix}
\usepackage{minitoc}
\newif\ifaddtoc

\addtocfalse

\ifaddtoc
    \usepackage[toc,page,header]{appendix}
    \usepackage{minitoc}

\fi

\newif\ifarxivversion 
\arxivversiontrue  %

\theoremstyle{plain}
\newtheorem{theorem}{Theorem}[section]
\newtheorem{proposition}[theorem]{Proposition}
\newtheorem{lemma}[theorem]{Lemma}
\newtheorem{corollary}[theorem]{Corollary}
\theoremstyle{definition}
\newtheorem{definition}[theorem]{Definition}

\theoremstyle{remark}

\input{macros}

\begin{document}

\author{%
\begin{tabular}{ccc}
    & & \\[1ex]
    \large{Michal Lukasik$^*$} & 
    \large{Lin Chen$^*$} & 
    \large{Harikrishna Narasimhan$^*$} \\
    \texttt{mlukasik@google.com} &
    \texttt{linche@google.com} & \texttt{hnarasimhan@google.com} \\[1ex]
    \large{Aditya Krishna Menon$^*$} & 
    \large{Wittawat Jitkrittum} &
    \large{Felix X. Yu} \\
    \texttt{adityakmenon@google.com} & 
    \texttt{wittawat@google.com} &
    \texttt{felixyu@google.com} \\[1ex]
    \large{Sashank J. Reddi} &
    \large{Gang Fu} & 
    \large{Mohammadhossein Bateni} \\
    \texttt{sashank@google.com} & 
    \texttt{thomasfu@google.com} & 
    \texttt{bateni@google.com} \\[1ex]
    & \large{Sanjiv Kumar} & \\
    & \texttt{sanjivk@google.com} & \\[2ex]
    & \large{Google Research} & \\[0ex]
\end{tabular}
}
\affiliation{}
\def\thefootnote{*}\footnotetext{Lead co-authors.}

\doparttoc %
\faketableofcontents %

\title{Bipartite Ranking From Multiple Labels: On Loss Versus Label Aggregation}

\maketitle

\input{body}

\bibliographystyle{plainnat}
\bibliography{references}

\clearpage
\appendix

\addcontentsline{toc}{section}{Appendix} %
\part{Appendix} %
\parttoc %

\input{appendix}

\end{document}

%% file: macros.tex
\usepackage[mathscr]{eucal}
\usepackage{amsmath}
\usepackage{soul}
\usepackage{enumitem}

\usepackage{colortbl}
\usepackage{multirow}
\allowdisplaybreaks

\usepackage[scaled=0.8]{FiraMono}

\newcommand{\AUCROC}{AUC}
\newcommand{\Dbp}{D}

\newcommand{\defEq}{\stackrel{.}{=}}

\newcommand{\argmax}{{\operatorname{argmax }}}

\renewcommand{\Pr}{\mathbb{P}}

\newcommand{\XCal}{\mathscr{X}}
\newcommand{\YCal}{\mathscr{Y}}

\newcommand{\Real}{\mathbb{R}}

\newcommand\independent{\protect\mathpalette{\protect\independenT}{\perp}}
\def\independenT#1#2{\mathrel{\rlap{$#1#2$}\mkern2mu{#1#2}}}

\newcommand{\1}{\mathbf{1}}

\renewcommand{\P}{\mathbf{P}}

\newcommand{\best}[1]{\cellcolor{gray!25}{{#1}}}

\newcommand{\bE}{\mathbb{E}}

\DeclareMathOperator{\auc}{AUC}

\newcommand{\sX}{\mathsf{X}}

\newcommand{\sY}[1]{\mathsf{Y}^{(#1)}}

\newcommand{\by}{\mathbf{y}}

%% file: body.tex
\input{abstract}

\section{Introduction}
\label{s:introduction}
\input{intro}

\section{Background and Notation}
\label{sec:problem}
\input{background}

\section{Bipartite Ranking With Multiple Labels}
\label{sec:multiobjective_intro}

\input{multiobjective_intro}

\section{The Loss and Label Aggregation Methods}
\label{sec:methods}
\input{multiobjective_methods}

\section{Bayes-Optimal Scorers Under Aggregation}
\input{random_binary}

\section{The Perils of Loss Aggregation}
\label{sec:deterministic}
\input{deterministic}

\section{Experimental Results}
\label{sec:experiments}

\input{experiments}

\section{Related Work}
\input{related}

\section{Conclusion and Future Work}
\input{conclusions}

%% file: abstract.tex
\begin{abstract}
Bipartite ranking is a fundamental supervised learning problem,
with
the goal 
of
learning
a ranking over instances with maximal 
\emph{area under the ROC curve}
(\emph{AUC}) 
against a \emph{single} binary target label.
However,
one may often observe \emph{multiple} binary target labels,
e.g., from distinct human annotators.
How can one synthesize such labels into a \emph{single} coherent ranking?
In this work, we formally analyze two approaches to this problem
---
\emph{loss aggregation} and \emph{label aggregation}
---
by
characterizing their \emph{Bayes-optimal} solutions.
We show that while
both approaches can yield 
Pareto-optimal solutions,
loss aggregation can exhibit \emph{label dictatorship}: 
one can inadvertently (and undesirably) favor one label over others.
This suggests that 
label aggregation can be preferable to loss aggregation,
which we empirically verify. 
\end{abstract}

%% file: intro.tex
Bipartite ranking is a fundamental supervised learning problem~\citep{Freund:2003,Cortes:2003,Agarwal:2005,Clemencon:2008,Kotlowski:2011,Menon:2016}, 
wherein the goal is to learn a ranker that orders ``positive'' instances over ``negative'' instances.
This is formalized as learning a ranker with maximal \emph{area under the ROC curve} (\emph{AUC})~\citep{Cortes:2003,Agarwal:2005,Krzanowski:2009},
and is arguably the simplest instantiation of \emph{learning to rank}~\citep{Liu:2009}.
Bipartite ranking has seen applications in practical problems ranging from
medical diagnosis~\citep{Swets:1988,Pepe:2003} to information retrieval~\citep{Ng:2000,Macskassy:2001},
and is the basis for several other supervised learning problems~\citep{Narasimhan:2013,reddi2021rankdistil,tang2022smooth}.

Bipartite ranking 
conventionally
assumes the existence of a \emph{single} binary label denoting whether or not an instance is ``positive''.
However,
in practice,
there may be \emph{multiple} binary labels,
each identifying a different set of ``positive'' instances.
For example, in information retrieval,
it is common for there to be distinct label sources (e.g., whether or not a user clicks on a document, whether or not a human rater deems a document to be relevant)~\citep{Svore2011}.
Similar
challenges 
arise in applications including
medical diagnosis (e.g., balancing opinions from multiple experts as to the presence of a certain condition~\citep{Verma:2023}),
recommendation systems (e.g., balancing click probability with diversity), 
and computational advertising (e.g., balancing click-through rates with conversion rates);
see Appendix~\ref{app:application_scenarios} for more discussion
of problems involving synthesizing multiple label sources.

Given multiple binary labels,
can one produce a \emph{single} coherent ranking over instances?
Addressing this requires formalising the \emph{goal} of such a ranking,
and then specifying the mechanics of \emph{achieving} this goal.
The \emph{multi-objective learning to rank} literature provides guidance on both points~\citep{Svore2011}.
Typically, the goal here is to 
find \emph{Pareto optimal} solutions,
i.e.,
solutions that are not dominated across all objectives~\citep{Ribeiro:2015}.
Towards achieving this,
previous works have considered two primary approaches:
(1)
\emph{loss aggregation} (also known as linear scalarization, or weighted sum), 
where 
the objective is a
weighted combination of per-label objectives~\citep{lin2019pareto};
(2) 
\emph{label aggregation}, 
where the objective is
formulated on a single target formed from suitable aggregation of the multiple labels,
such as a weighted combination~\citep{Carmel2020}.

Empirically,
both loss and label aggregation have proven
successful for multi-objective ranking problems.
Theoretically,
however, there has been limited 
analysis providing guidance on the following natural question: 
\emph{are there reasons to favor label over loss aggregation, or vice-versa}?

\input{tbl_summary_binary_stochastic_vs_deterministic}

In this work, 
we study this question in the context of bipartite ranking.
We
study the \emph{Bayes-optimal} scorers for both aggregation approaches,
which represent the theoretical minimizers given full access to the underlying statistical distributions.
We find that both methods \emph{broadly} yield comparable optimal rankings,
which furthermore are Pareto optimal.
However, 
we explicate how
Pareto optimality --
while providing a foundational concept for multi-objective problems
--
may by itself be insufficient for a desirable practical solution.
Indeed,
upon closer inspection, 
we demonstrate that loss aggregation 
can 
lead to a \emph{label dictatorship} phenomenon,
wherein certain labels are implicitly favored over others.
This suggests that label aggregation can be preferred over loss aggregation,
which we validate empirically.

In summary,
our contributions are (cf.\
Table~\ref{tbl:summary_binary_stochastic_vs_deterministic}):
\begin{enumerate}[label=(\roman*),itemsep=0pt,topsep=0pt,leftmargin=16pt]
    
    \item We formalize
    the problem of bipartite ranking from multiple labels (\S\ref{sec:multiobjective_intro}),
    and the instantiation of
    loss aggregation
    and label aggregation in this setting
    (\S\ref{sec:methods});

    \item We characterize the \emph{Bayes-optimal solutions} to both
    loss (\S\ref{s:linear_scalarization})
    and label aggregation (\S\ref{s:label_agg}), 
    and 
    establish Pareto-optimality of (suitable instantiations of) 
    both methods;

    \item We show that loss aggregation can lead to an
    undesirable \emph{label dictatorship} phenomenon (\S\ref{sec:deterministic}),
    and empirically validate this
    on synthetic and real-world datasets (\S\ref{sec:experiments}). 
\end{enumerate}
Our study serves to analyze the theoretical properties of loss and label aggregation, 
rather than 
constructing novel algorithms. 
This analysis, however, yields practical insights for 
selecting amongst these methods.

%% file: tbl_summary_binary_stochastic_vs_deterministic.tex
\begin{table*}
    \centering
    \resizebox{0.98\linewidth}{!}{%
    \begin{tabular}{@{}lcllllll@{}}
        \toprule
        \textbf{Objective} & 
        \textbf{Defn.} & 
        \textbf{Cost} $c_{\bar{y}\bar{y}'}$ &
        \textbf{Stochastic labels (Any $K$) } &
        \textbf{Deterministic labels ($K=2$)} &
        \textbf{Pareto-opt} &
        \textbf{Dictatorship} &
        \textbf{Prop.}
        \\
        \toprule
        \toprule
        Loss aggregation
        &
        \eqref{eqn:auc-multi-dist}
        &
        N/A
        &
        $ \sum_{k} \alpha^{(k)} \cdot\eta^{(k)}(x)$ 
        &
         $ \alpha^{(1)} \cdot\eta^{(1)}(x) + \alpha^{(2)} \cdot\eta^{(2)}(x)$ 
        &
        \checkmark
        &
        \checkmark
        &
        \ref{lemm:bayes-opt-auc-multi-dist},
        \ref{prop:dictatorship}
        \\
        \midrule
        Label aggregation 
        &
        \multirow{2}{*}{\eqref{eqn:auc-label-agg}}
        &
        $1$
        &
        No closed form
        &
         $ \eta^{(1)}(x) + \eta^{(2)}(x)$ 
        &
        \ding{55}
        &
        \ding{55}
        &
        \ref{prop:pareto-label-agg}, \ref{prop:label-agg-deterministic}
        \\
        with $\bar{\mathsf{Y}} = \sum_{k} \mathsf{Y}^{(k)} $
        &
        &
        $\1( \bar{y} > \bar{y}'  ) \cdot | \bar{y} - \bar{y}' |$
        &
         $\sum_{k} \eta^{(k)}( x )$
        &
          $ \eta^{(1)}(x) + \eta^{(2)}(x)$ 
        &
        \checkmark
        &
        \ding{55}
        &
        \ref{lemm:bayes-opt-auc-label-agg}
        \\
        \bottomrule
    \end{tabular}%
    }
    \caption{Bayes-optimal scorer for 
    two
    approaches to
    bipartite ranking
    with $K$ binary labels. Here, 
    we assume that we have 
    random variables
    $( \mathsf{X}, \mathsf{Y}^{(1)}, \ldots, \mathsf{Y}^{(K)} )$
    representing
    instances and labels 
    drawn from some joint distribution,
    with $\eta^{(k)}( x ) \defEq \P( \mathsf{Y}^{(k)} = 1 \mid \mathsf{X} = x )$ denoting the marginal class-probability function for each label. 
    Suitable instantiations of both loss aggregation and label aggregation satisfy a notion of Pareto-optimality;
    however,
    the loss aggregation approach results in an undesirable ``label dictatorship'' phenomenon,
    wherein one label dominates the other (see~\S\ref{sec:dictatorship}).
    The label aggregation objective in \eqref{eqn:auc-label-agg} requires the specification of pairwise costs $c_{\bar{y}\bar{y}'}$.  The constants
    $\alpha^{(k)} = { a_k } / ( {\pi^{(k)} \cdot (1 - \pi^{(k)})} )$ arise in Proposition~\ref{prop:dictatorship}, where $\pi^{(k)} \defEq \P( \mathsf{Y}^{(k)} = 1 )$ denotes the label priors.}
    \label{tbl:summary_binary_stochastic_vs_deterministic}
\end{table*}

%% file: background.tex
Learning to rank problems seek to
learn a scorer that orders instances according to an underlying utility score~\citep{Liu:2009}.
The
\emph{area under the ROC curve} (\emph{\AUCROC{}}) 
is
a traditional metric used to evaluate the efficacy of 
such a ranking~\citep{Cortes:2003, Agarwal:2005, Clemencon:2008, Menon:2016}. 
Below, we define \AUCROC{} for problems with binary and multi-class labels.

\subsection{Bipartite \AUCROC{}}

Consider a 
supervised learning problem with
input space $\mathscr{X}$, binary labels $\mathscr{Y} = \{ 0, 1 \}$,
and
distribution
$D$
over $\mathscr{X} \times \mathscr{Y}$.
Let $(\mathsf{X}, \mathsf{Y})$ be random variables distributed according to $D$.
Denote by $\eta(x) \defEq \P(\mathsf{Y}=1|\mathsf{X}=x)$ the 
\emph{class-probability function} and $\pi \defEq \P(\mathsf{Y}=1) $ the positive class \emph{prior}.

Our goal is to learn  a scorer%
$f \colon \XCal \to \Real$ that ranks the positive examples 
(i.e., those with $\mathsf{Y} = 1$) 
over the negative ones,
as quantified by the area under the ROC curve, or \AUCROC{}.

\begin{definition}[{\citet{Agarwal:2009}}]
\label{lemm:auc-bipartite}
For any scorer $f \colon \XCal \to \Real$ and distribution $D$, the (bipartite) \AUCROC{} is: 
\begin{align}
    \label{eqn:auc-bp-conditional}
    {\rm AUC}( f; D ) &\defEq \mathbb{E}_{\mathsf{X}, \mathsf{X}'} \left[ H( f( \mathsf{X} ) - f( \mathsf{X}' ) ) \mid \mathsf{Y} > \mathsf{Y}' \right]\\
    \nonumber
    H( z ) &\defEq \1( z > 0 ) + \frac{1}{2} \cdot \1( z = 0 ).
\end{align}
\end{definition}

Intuitively, 
the \AUCROC{} is the fraction of pairs $(x, x') \in \XCal \times \XCal$ with positive and negative labels wherein $f$ scores the positive over the negative, with a reward of $0.5$ for ties.

Given 
some fixed distribution $D$,
a \emph{Bayes-optimal} scorer $f^*$
is one
that achieves
the \emph{highest possible} \AUCROC{};
i.e., ${\rm AUC}( f^*; D ) \geq {\rm AUC}( f; D )$ for any $f \colon \XCal \to \Real$.
One may employ 
the Neyman-Pearson lemma~\citep{Lehmann:2005}
to establish these scorers
closely follow 
$\eta( x )$.

\begin{proposition}[\citet{Clemencon:2008,Menon:2016}]
\label{lemm:auc-bipartite-opt}
For any distribution $D$, any Bayes-optimal
\AUCROC{} scorer $f^*$  satisfies
\begin{align*}
( \forall x, x' \in \XCal ) \, \eta( x ) > \eta( x' ) \implies f^*( x ) > f^*( x' ),
\end{align*}
or equally, $\eta$ is any non-decreasing transformation of $f^*$.
\end{proposition}

We remark here that 
neither the AUC nor its  Bayes-optimal scorer depend on the class prior $\pi$.
Thus, the (population) AUC is invariant to the amount of label skew, supporting its common usage as a metric in problems characterized by label imbalance~\citep{Ling:1998,Menon:2013}.

\subsection{Multipartite \AUCROC{}}

One may extend bipartite ranking to \emph{ordinal multi-class} labels 
$\mathscr{Y} = \{ 1, 2, \ldots, L \}$,
wherein higher label values denote higher presence of a certain attribute (e.g., star ratings denoting user's item preferences). 
In this case, our goal is to learn a scorer that ranks examples with higher labels over examples with lower labels.
Further, one may apply variable \emph{costs} on mis-ranking of pairs with different labels.

Formally, let $D_{\rm mp}$ denote a distribution over $\XCal \times \YCal$.
As before, we denote the conditional-class probability function by $\eta_y(x) = \P(\mathsf{Y}=y\mid \mathsf{X}=x)$, and the class priors by $\pi_y = \P(\mathsf{Y}=y)$. 
Further, let $c_{yy'} \geq 0$ denote the cost of scoring an instance with label $y \in \YCal$ below an instance with label $y' \in \YCal$.
The following is an adaptation of the bipartite \AUCROC{} (Definition~\ref{lemm:auc-bipartite}) to multi-class problems.

\begin{definition}[\citet{Uematsu:2015}]\label{def:multipartite-auc}
For any scorer $f \colon \XCal \to \Real$, distribution $D_{\rm mp}$, and costs 
$\{ c_{yy'} \geq 0  \colon y, y' \in \YCal \}$, the multi-partite \AUCROC{} is: 
\ifarxivversion
    \begin{equation}
    \label{eqn:auc-mp}
    {\rm AUC}_{\rm mp}( f; D_{\rm mp} ) \defEq 
    \mathbb{E}_{\mathsf{X}} \mathbb{E}_{\mathsf{X}'} \left[ c_{\mathsf{Y} \mathsf{Y'}} \cdot H( f( \mathsf{X} ) - f( \mathsf{X}' ) ) \mid \mathsf{Y} > \mathsf{Y}' \right].\nonumber
    \end{equation}
\else
\begin{align}
    \label{eqn:auc-mp}
    \lefteqn{{\rm AUC}_{\rm mp}( f; D_{\rm mp} ) \,\defEq }
    \\
    & 
    \hspace{1cm}
    \mathbb{E}_{\mathsf{X}} \mathbb{E}_{\mathsf{X}'} \left[ c_{\mathsf{Y} \mathsf{Y'}} \cdot H( f( \mathsf{X} ) - f( \mathsf{X}' ) ) \mid \mathsf{Y} > \mathsf{Y}' \right].\nonumber
\end{align}
\fi
\end{definition}

When $L = 2$, the objective reduces to a scaling of~\eqref{eqn:auc-bp}.
Unlike the bipartite case, the multipartite \AUCROC{} does \emph{not} admit a tractable Bayes-optimal scorer in general.
An exception is when either $L = 3$, or the costs satisfy the following \emph{scale condition}:
for suitable constants $w_y, s_y \geq 0$,
\begin{equation}
\label{eqn:cost-condition}
   ( \forall y, y' \in \mathscr{Y} ) \, c_{yy'} = w_y \cdot w_{y'} \cdot ( s_{y} - s_{y'} ) \cdot \1( y > y' ),
\end{equation}
For example,~\eqref{eqn:cost-condition} holds when 
$c_{yy'} = ( {y} - {y'} ) \cdot \1( y > y' )$.
Under such a condition, we have the following.

\begin{proposition}[{\citet[Theorem 3]{Uematsu:2015}}]
\label{lemm:auc-multipartite-opt}
Suppose $L = 3$, or the costs $\{ c_{yy'} \geq 0 \colon y, y' \in \YCal \}$ satisfy the scale condition~\eqref{eqn:cost-condition}.
For any distribution $D_{\rm mp}$,
any Bayes-optimal multi-partite \AUCROC{} scorer $f^*$ satisfies
\begin{align*}
    ( \forall x, x' \in \XCal ) \, \beta( x ) &> \beta( x' ) \implies f^*( x ) > f^*( x' ) \\
    \beta( x ) &\defEq \frac{\sum_{i = 2}^{L} c_{1, i} \cdot \eta_{i}( x )}{\sum_{i = 1}^{L - 1} c_{i, L} \cdot \eta_{i}( x )},
\end{align*}
or equally, $\beta$
is any non-decreasing transformation of $f^*$.
\end{proposition}

%% file: multiobjective_intro.tex
We now formalize our setting of
interest
---
bipartite ranking from \emph{multiple} labels
---
and the core goal of identifying a \emph{Pareto optimal} solution with respect to these labels.

\subsection{Formal setup}\label{sec:formal_setup}

Consider a setting
where we have access to \emph{multiple} 
labels for each instance $x \in \XCal$, and would like to produce a single coherent ranking over instances.
For simplicity, we begin 
by assuming
each individual label is \emph{binary}. 
Formally,
for integer $K \geq 2$,
$( \mathsf{X}, \mathsf{Y}^{(1)}, \ldots, \mathsf{Y}^{(K)} )$
be random variables distributed according to
a joint distribution
$D^{\text{jnt}}$
over $\XCal \times \{ 0, 1 \}^{K}$.
Let $\mu$ denote the marginal distribution of $\mathsf{X}$.
For each $k \in [ K ]$, we have an induced marginal distribution
$D^{(k)}$ over $( \mathsf{X}, \mathsf{Y}^{(k)} )$.
Let
$\eta^{(k)}( x ) \defEq \P( \mathsf{Y}^{(k)} = 1 \mid \mathsf{X} = x )$ 
denote the marginal class-probability function of each label $k \in [ K ]$,
and
$\pi^{(k)} \defEq \P(\mathsf{Y}^{(k)} = 1)$ the class prior.

Our goal remains to produce a \emph{single} 
scorer 
$f \colon \XCal \to \Real$.
To do so, we must specify a concrete metric for assessing $f$.

\subsection{Pareto optimal per-label AUC maximisation}

One natural summary of $f$'s 
performance
is the \emph{per-label AUC vector},
i.e.,
$( {\rm AUC}( f; D^{(1)} ), \ldots, {\rm AUC}( f; D^{(K)} ) )$.
The problem of synthesizing such a vector into a single metric falls within the purview of multi-objective optimization~\citep{Ehrgott:2005}.
Adapting
a standard goal from this literature, 
it is natural to seek scorers 
that are \emph{Pareto optimal}~\citep{Ehrgott:2005} with respect to the 
per-label AUCs. 

\begin{definition}[Pareto dominance]
We say a scorer $f \colon \XCal \to \Real$ \emph{Pareto dominates} another $g \colon \XCal \to \Real$ with respect to distributions $\{ D^{(k)} \}_{k \in [K]}$
if and only if:
\begin{enumerate}[label=(\arabic*),itemsep=0pt,topsep=0pt]
    \item $\forall k \in [ K ] : \auc(g; D^{(k)}) \geq \auc(f; D^{(k)})$
    \item $\exists k \in [ K ] : \auc(g; D^{(k)}) > \auc(f; D^{(k)})$.
\end{enumerate}
\end{definition}

\begin{definition}[Pareto optimality]
For any $\{ D^{(k)} \}_{k \in [K]}$,
the set of Pareto optimal scorers $\mathscr{F}_{\textnormal{PFS}}$ 
comprises scorers
that are \emph{not} Pareto dominated by any other scorer.
\end{definition}

Essentially, a scorer is Pareto optimal if no other scorer dominates it across at least one AUC objective, while not harming it across all AUC objectives.

\subsection{Does Pareto optimality suffice?}
\label{sec:pareto-suffice}

\begin{table*}
  \centering
  \small
    \resizebox{\linewidth}{!}{
    \begin{tabular}{@{}p{0.8in}p{4.4in}p{3.8in}@{}}
    \toprule
    \textbf{Query} & {\textbf{Low relevance \& High engagement document}} & {\textbf{High relevance \& Low engagement document}}\\
    \midrule
    \cellcolor{white}{\tt"What is another name for rust?"} & \cellcolor{blue!10}{First of all rust is formed when iron is exposed to both oxygen and water/ water vapopurs. The formula for rust is Fe203.xH20. Now the x varies which determines the extent to which rust is formed. Basically Rust is formed throughout the surface of the iron thus preventing rusting of the inner layers.} &
    \cellcolor{yellow!10}{Rust is the result of the oxidation of iron. The most common cause is prolonged exposure to water. Any metal that contains iron, including steel, will bond with the oxygen atoms found in water to form a layer of iron oxide, or rust. Rust will increase and speed up the corrosion process, so upkeep is important.here are two ways you can use a potato to remove rust: 1 Simply stab the knife into potato and wait a day or overnight. 2 Slice a potato in half, coat the inside with a generous portion of baking soda, and go to town on the rusted surface with the baking soda-coated potato} \\
    \midrule
     \cellcolor{white}{\tt"How much does it cost to build a deck with a hot tub?"} & 
    \cellcolor{blue!10}{Once you determine what kind of decking materials you'll need to build the structure, it's time to get down to price. While the average cost to build a deck averages between \$4,000 and \$10,000, that doesn't account for the materials.Here is the average cost of each decking material, broken down by average price range per board: It's important to know what board sizes you'll need to purchase for your deck.nce you determine what kind of decking materials you'll need to build the structure, it's time to get down to price. While the average cost to build a deck averages between \$4,000 and \$10,000, that doesn't account for the materials.} &
    \cellcolor{yellow!10}{1 Pre-fabricated hot tubs cost less, but are still priced between \$3,000 and \$8,000 (for the smaller-sized models). 2 Although you give up some features and styling with inflatable, you can get a good-quality inflatable hot tub starting at about \$500. 3 There is quite a difference in price here.	} \\
    \\
    \bottomrule
    \end{tabular}
    }
    \vspace{-0.2cm}
    \caption{An illustration of the trade-off between engagement and relevance across queries and documents from the MS MARCO dataset~\citep{bajaj2018msmarcohumangenerated}. 
    A purely relevance-driven approach recommends the documents from the rightmost column, which contains the correct answer,
    albeit with an unclear presentation.
    Conversely, a purely engagement-driven approach recommends the documents from the first column, which may superficially match keywords, but lack the correct answer.
    While both solutions may be Pareto optimal,
    in practice,
    one may favor only one of these solutions;
    further, the particular choice can vary depending on the query. Therefore, a solution which \emph{globally} favors one signal over the other (i.e., where one label is a \emph{dictator}) may be undesirable.}
    \label{table:example_msmarco1}
\end{table*}

Pareto optimality is a necessary, but not sufficient condition for a solution to be practically useful:
a model that aggressively optimizes one objective at the complete expense of others, while being Pareto optimal, can be 
undesirable.

For example,
consider a document retrieval problem, 
with the goal of
satisfying various aspects of a user query;
e.g., an ideal system should return
documents
that are both \emph{relevant} to the query and likely to \emph{engage} the user \citep{Que2Engage}.
Relevance can be assessed through annotation from a human or machine expert, 
while engagement can be measured through metrics such as click-through rate or revenue.

In practice, the goals of relevance and engagement may be at odds with each other.
To illustrate this point, consider queries
from the MS MARCO benchmark dataset~\citep{bajaj2018msmarcohumangenerated}.
We predict the engagement and relevance of different documents for a query by prompting the Gemini model~\citep{Anil:2023} 
(see Table~\ref{table:prompts_msmarco} in Appendix for details). 
Illustrative queries, documents, and predicted engagement and relevance scores are reported in Table~\ref{table:example_msmarco1}. %
For the query {\tt``What is another name for rust?''}, 
a document with high relevance but low engagement
contains the right answer, albeit with an obscure presentation.
Conversely,
a document with high engagement but low relevance
superficially matches keywords, but lacks the correct answer.

On the other hand, for the query {\tt"How much does it cost to build a deck with a hot tub?"}, 
a document with high relevance but low engagement
addresses the specific question of the hot tub pricing, but ignores the more general context of the price for building the deck.
A document with high engagement but low relevance
misses the hot tub aspect, 
but does answer the general question of deck pricing.

These examples highlight 
a common practical challenge: 
a trade-off often exists between different objectives (in this case, relevance and engagement). 
Further,
while scorers focusing exclusively on one metric could be Pareto optimal 
-- 
specifically, if no single document excels at both metrics simultaneously
--
such 
solutions
may be practically undesirable.
Indeed,
solely optimizing relevance at the expense of engagement may be desirable for the first query (where the document contains the correct answer), 
but less desirable for the second query (where the query intent is  open-ended). 

Therefore, alongside ensuring Pareto optimality, it is crucial to assess whether a solution does not overly favor one of the labels,
e.g. via the gap between the per-objective AUC scores (as we consider in \S\ref{sec:experiments}).
This motivates a deeper investigation into the precise forms of Bayes-optimal solutions for different aggregation approaches, allowing us to understand their inherent behaviors and potential biases.

%% file: multiobjective_methods.tex
We now formalize two natural approaches for 
bipartite ranking with multiple binary labels.
These aggregate either the \emph{losses} over multiple labels, or the \emph{labels} themselves.

\subsection{Loss aggregation}

A common strategy in the multi-objective optimization literature to achieve Pareto optimal solutions is to optimize a linear combination of the objectives \cite{ruchte2021scalableparetoapproximationdeep}. 
Such an approach is typically referred to as 
\emph{linear scalarization},
and may be seen as performing \emph{loss aggregation}.
In our case, this amounts to maximizing 
\begin{align*}
    \sum_{k \in [K]} a_k \cdot {\rm AUC}( f; D^{(k)} ),
\end{align*}
for mixing coefficients $a_1,\ldots,a_k > 0$.
This is equivalent to the following \emph{loss aggregated} AUC objective. 

\begin{definition}[Loss aggregation]
\label{def:linear-scalarization}
For any scorer $f \colon \XCal \to \Real$, 
distribution $D^{\text{jnt}}$ over $( \mathsf{X}, \mathsf{Y}^{(1)}, \ldots, \mathsf{Y}^{(K)} )$,
and weights $\{ a_k > 0 \}_{k \in [K]}$, 
the \emph{loss aggregated AUC} is: 
\ifarxivversion
  \begin{equation} \label{eqn:auc-multi-dist}
  {\rm AUC}_{\rm LoA}( f; D^{\text{jnt}} )
  \defEq 
  \sum_{k \in [K]} \mathbb{E}_{\mathsf{X}, \mathsf{X}'} \Big[  a_k \cdot H( f( \mathsf{X} ) - f( \mathsf{X}' ) ) \Big| \mathsf{Y}^{(k)} > \mathsf{Y}'^{(k)} \Big]
  \end{equation}
\else
  \begin{align}
  \label{eqn:auc-multi-dist}
  & {\rm AUC}_{\rm LoA}( f; D^{\text{jnt}} )
  \\
  &\defEq  \sum_{k \in [K]} \mathbb{E}_{\mathsf{X}, \mathsf{X}'} \Big[  a_k \cdot H( f( \mathsf{X} ) - f( \mathsf{X}' ) ) \Big| \mathsf{Y}^{(k)} > \mathsf{Y}'^{(k)} \Big].
  \nonumber
  \end{align}
\fi
\end{definition}
For brevity,
we subsequently omit the dependence 
of ${\rm AUC}_{\rm LoA}( f )$ on $D^{\text{jnt}}$.
Given a finite sample $\{ ( x^{(i)}, y^{(i, 1)}, \ldots, y^{(i, k)} ) \}_{i \in [ N ]}$ drawn from $D^{\text{jnt}}$,
the empirical loss aggregated AUC is
\begin{align*}
    \widehat{{\rm AUC}}_{\rm LoA}( f ) \propto \sum_{k \in [K]}  \sum_{(i, j) \in P^{(k)}} a_k \cdot 
    H( f( x^{(i)} ) - f( x^{(j)} ) ),
\end{align*}
where
$P^{(k)} \defEq \{ (i, j ) \in [ N ] \times [ N ] \colon \1( y^{(i, k)} > y^{(j, k)} ) \}$.

\subsection{Label aggregation}

Weighting the individual per-label AUCs is conceptually simple.
An alternative approach is to combine the $K$ labels 
$( \mathsf{Y}^{(1)}, \ldots, \mathsf{Y}^{(K)} )$
via some \emph{aggregation function} $\psi$ to obtain a single new \emph{aggregated} label $\bar{\mathsf{Y}}$ \cite{Svore2011,Agarwal2011,Dai2011,Carmel2020,wei2023aggregate}.
Given such an aggregated label,
one may then maximize the \emph{multi-partite} AUC~\eqref{eqn:auc-mp} on $\bar{\mathsf{Y}}$, as follows:
\begin{definition}[Label aggregation]
\label{def:label-aggregation}
For any scorer $f \colon \XCal \to \Real$, 
distribution $D^{\text{jnt}}$ over $( \mathsf{X}, \mathsf{Y}^{(1)}, \ldots, \mathsf{Y}^{(K)} )$,
aggregation function $\psi: \YCal^K \rightarrow \bar{\YCal}$, and costs $\{ c_{\bar{y} \bar{y}'} \geq 0 \}_{\bar{y}, \bar{y'} \in \bar{\mathscr{Y}}}$, 
the \emph{label aggregated AUC} is: 
\ifarxivversion
  \begin{equation} \label{eqn:auc-label-agg}
  {\rm AUC}_{\rm LaA}( f; D^{\text{jnt}} )
  \defEq 
  \mathbb{E}_{\mathsf{X}, \mathsf{X}'} \left[ c_{\bar{\mathsf{Y}} \bar{\mathsf{Y}}'} \cdot H( f( \mathsf{X} ) - f( \mathsf{X}' ) ) \mid \bar{\mathsf{Y}} > \bar{\mathsf{Y}}' \right]
  \end{equation}
\else
  \begin{align}
    \label{eqn:auc-label-agg}
    & {\rm AUC}_{\rm LaA}( f; D^{\text{jnt}} ) %
    \\ 
     \defEq{}&  \mathbb{E}_{\mathsf{X}, \mathsf{X}'} \left[ c_{\bar{\mathsf{Y}} \bar{\mathsf{Y}}'} \cdot H( f( \mathsf{X} ) - f( \mathsf{X}' ) ) \mid \bar{\mathsf{Y}} > \bar{\mathsf{Y}}' \right]
    \nonumber 
  \end{align}
\fi
where 
$\bar{\mathsf{Y}} = \psi(\mathsf{Y}^{(1)}, \ldots, \mathsf{Y}^{(K)})$ and $\bar{\YCal} \subseteq [K]$.
\end{definition}

Given a finite sample $\{ ( x^{(i)}, y^{(i, 1)}, \ldots, y^{(i, k)} ) \}_{i \in [ N ]}$ drawn from $D^{\text{jnt}}$,
and
$\bar{y}^{(i)} \defEq \sum_{k \in [ K ]} y^{(i, k)}$,
the empirical label aggregated AUC is
\begin{align*}
    \widehat{{\rm AUC}}_{\rm LaA}( f ) \propto \sum_{(i, j) \in P} c_{\bar{y}^{(i)}, \bar{y}^{(j)}} \cdot H( f( x^{(i)} ) - f( x^{(j)} ) ),
\end{align*}
where
$P \defEq \{ (i, j ) \in [ N ] \times [ N ] \colon \1( \bar{y}^{(i)} > \bar{y}^{(j)} ) \}$.

There are several natural choices of aggregation function.
{
Per \S\ref{sec:formal_setup}, we consider $K$ individual binary labels $\mathsf{Y}^{(k)} \in \{0,1\}$. 
One approach is to sum these binary labels: $\psi(\mathsf{Y}^{(1)},...,\mathsf{Y}^{(K)})=\sum_{k\in[K]}\mathsf{Y}^{(k)} \in \{0,1,...,K\}$. 
This results in an integer ordinal value, representing the \textit{count} of positive labels for an instance $x$. 
}
Another natural label aggregation mechanism could be to take the product of the labels: $\psi(\mathsf{Y}^{(1)}, \ldots, \mathsf{Y}^{(K)})= \prod_{k \in [K]} \mathsf{Y}^{(k)} \in \{ 0, 1 \}$.

%% file: random_binary.tex
We now 
characterize the set of Bayes-optimal scorers for
loss aggregation (Definition \ref{def:linear-scalarization}) and label aggregation (Definition \ref{def:label-aggregation}).
This shall be a step towards understanding the solutions provided by 
each method.

\subsection{Bayes-optimal scorers for loss aggregation}
\label{s:linear_scalarization}
We first show that any scorer that is optimal for loss aggregation is also Pareto optimal; this follows straight-forwardly from~\citet{ruchte2021scalableparetoapproximationdeep}.

\begin{proposition}
\label{prop:loss-aggregated-pareto}
For any  distributions $\{ D^{(k)} \}_{k \in [K]}$ and mixing weights $\{ a_k \geq 0 \}_{k \in [K]}$, a scorer $f^*: \XCal \rightarrow \mathbb{R}$ that maximizes the loss aggregated AUC in \eqref{eqn:auc-multi-dist} is Pareto optimal.
\end{proposition}

As noted in Section \ref{sec:multiobjective_intro}, Pareto optimality alone may not suffice to assess the usefulness of a solution. We therefore take a closer look at the form of the Bayes-optimal scorer.

\begin{proposition}
\label{lemm:bayes-opt-auc-multi-dist}
For any  distributions $\{ D^{(k)} \}_{k \in [K]}$ and mixing weights $\{ a_k \geq 0 \}_{k \in [K]}$, any Bayes-optimal scorer $f^*$ for the loss aggregated AUC satisfies
\begin{align}
        \gamma( x ) &> \gamma( x' ) \implies f^*( x ) > f^*( x' ) 
        \nonumber
        \\
        \label{eqn:bayes-opt-ls}
        \gamma( x ) &\defEq \frac{1}{K} \sum_{k \in [K]} \frac{a_k}{\pi^{(k)} \cdot (1 - 
        \pi^{(k)})} \cdot \eta^{(k)}( x ),
\end{align}
or equally, 
$\gamma$
is some non-decreasing transformation of
$f^*$.
\end{proposition}

The final expression is intuitive:
the optimal scorers involve a weighted sum of individual class-probability functions $\eta^{(k)}$. 
In fact, when $K = 1$, this reduces to the standard result for the AUC (Lemma~\ref{lemm:auc-bipartite-opt}).

\subsection{Bayes-optimal scorers for label aggregation}
\label{s:label_agg}

Unlike loss aggregation, label aggregation may \emph{not} 
always
produce Pareto optimal solutions. 
We show this below with additive label aggregation function and with costs $c_{\bar{y}\bar{y}'} = 1$.

\begin{proposition}
\label{prop:pareto-label-agg}
Suppose $\psi(\mathsf{Y}^{(1)}, \ldots, \mathsf{Y}^{(K)}) = \sum_{k}  \mathsf{Y}^{(k)}$, and  $c_{\bar{y}\bar{y}'} = 1, \forall \bar{y}, \bar{y}' \in \bar{\YCal}$. There exists a set of distributions $\{ D^{(k)} \}_{k \in [K]}$, mixing weights $\{ a_k \geq 0 \}_{k \in [K]}$, and a scorer $f^*: \XCal \rightarrow \mathbb{R}$ such that $f^*$   maximizes the label aggregated AUC in \eqref{eqn:auc-label-agg}, but is \emph{not} Pareto optimal.
\end{proposition}

Interestingly, this issue can be remedied with a simple change to the AUC objective. Specifically, we choose the  misranking costs  to be $c_{\bar{y} \bar{y}'} = \1( \bar{y} > \bar{y}'  ) \cdot | \bar{y} - \bar{y}' |$ in~\eqref{eqn:auc-label-agg}, where the cost of misranking a pair of instances scales linearly with the difference in their aggregated labels.
We can then show that the optimal scorers are obtained through simple summation of the class probability functions $\eta^{(k)}(x)$; 
thus, the Bayes-optimal scorers are also Pareto optimal. %
\begin{proposition}
\label{lemm:bayes-opt-auc-label-agg}
Suppose $\psi(\mathsf{Y}^{(1)}, \ldots, \mathsf{Y}^{(K)}) = \sum_{k}  \mathsf{Y}^{(k)}$, and 
$c_{\bar{y} \bar{y}'} = \1( \bar{y} > \bar{y}'  ) \cdot | \bar{y} - \bar{y}' |$.
For any  distributions $\{ D^{(k)}\}_{k \in [K]}$, 
any Bayes-optimal scorer $f^*$ for the label-aggregated AUC 
in \eqref{eqn:auc-label-agg}
satisfies
\begin{align}
        \gamma( x ) &> \gamma( x' ) \implies f^*( x ) > f^*( x' ) 
        \nonumber
        \\
        \gamma( x ) &\defEq 
        \sum_{k \in [K]} \eta^{(k)}( x );
        \label{eq:bayes-opt-la}
\end{align}
or equally, 
$\gamma$
is some non-decreasing transformation of
$f^*$. Furthermore, $f^*$ is \emph{Pareto optimal} w.r.t.\ $\{ D^{(k)}\}_{k \in [K]}$.
\end{proposition}

{
The simple form of~\eqref{eq:bayes-opt-la} 
critically relies on choosing costs $c_{\bar{y} \bar{y}'} = \1( \bar{y} > \bar{y}'  ) \cdot | \bar{y} - \bar{y}' |$ in the label aggregated AUC objective~\eqref{eqn:auc-label-agg}. As established by~\citet{Uematsu:2015} and utilized in our proof (\cref{sec:proofs}), these costs ensure that
the optimal scorer
follows $\mathbb{E}[\bar{\mathsf{Y}}\mid \sX=x]$, which equals $\sum_k \eta^{(k)}(x)$ by linearity of expectation. 
Other costs (e.g., uniform $c_{\bar{y}\bar{y}^{\prime}}=1$, per~\cref{prop:pareto-label-agg}) do not generally lead to this direct summation of $\eta^{(k)}$ as the optimal scorer.}

We also note that for an \emph{arbitrary} set of costs $c_{\bar{y} \bar{y}'}$,
the Bayes-optimal scorer will in general \emph{not} admit a tractable closed-form solution.
This is not surprising: indeed, even when $K = 1$, the multi-partite AUC does not have a tractable 
closed-form scorer in general~\cite{Uematsu:2015}.
However, for the specific choice of \emph{uniform costs} $c_{\bar{y} \bar{y}'} = 1$, we present scorers that are \emph{asymptotically} optimal for ${\rm AUC}_{\rm LA}$ (in the limiting case of $K \rightarrow \infty$) under some  distributional assumptions; see Theorem~\ref{thm:auc-gap} and Corollary~\ref{cor:auc-gap} in \S\ref{app:additional-results}.

\subsection{Contrasting the Bayes-optimal scorers} 
Upon closer inspection, the optimal solution to loss aggregation in~\eqref{eqn:bayes-opt-ls} reveals a subtlety: 
for $K>1$, even for a uniform weighting of the individual label AUCs (i.e., $a_{k}=1,\forall k$), the optimal scorers depend on the underlying marginal class priors $\pi^{(k)}$. 
Conceptually,
this is a surprise:
indeed, it is in contrast to the standard behaviour when $K = 1$,
wherein the class priors do \emph{not} influence the AUC optimizer. 

Practically,
this suggests that
the choice of the mixing weights $a_k$ determines the extent to which the optimal scorer favors one label over the others.
More precisely,
the optimal scorer \emph{favors labels that are marginally skewed} ($\pi^{(k)}$ is away from $0.5$), 
over labels that are balanced ($\pi^{(k)} \approx 0.5$).
{As we discuss in \S\ref{sec:dictatorship}, this can lead to an objective that inadvertently favors certain labels based purely on their dataset-wide imbalance, rather than their instance-specific information content, potentially leading to a ``label dictatorship'' (see Appendix~\ref{app:ir_dictatorship_example} for an illustrative example).}
This is unexpected: 
one might expect that a uniform weighting ($a_k = 1$) would induce scorers that treat all labels equitably.

{It is critical to distinguish the 
\emph{marginal} probabilities $\pi^{(k)} \defEq \mathbf{P}(\sY{k}=1)$
from the \emph{instance-conditional} probabilities 
$\eta^{(k)}(x) \defEq \mathbf{P}(\sY{k}=1\mid \sX=x)$.
A label can be perfectly balanced marginally (i.e., $\pi^{(k)}=0.5$) 
while being completely noise-free conditionally (i.e., $\eta^{(k)}(x) \in \{0,1\}$ for all $x$).
While weighting
by the inverse marginal skew 
$1/[\pi^{(k)}(1-\pi^{(k)})]$
might appear statistically natural, 
this 
may
not align with the 
conditional quality of the signals.}

By contrast,
under label aggregation,
the optimal scorers equally balance each $\eta^{(k)}$. 
Crucially, the optimal scorers do \emph{not} incorporate additional weighting under the hood, and 
thus 
\emph{avoid favoring one set of labels over the others}.
{Further analysis (see~\cref{app:further_label_agg_analysis}) shows this optimal scorer (\cref{lemm:bayes-opt-auc-label-agg}) behaves predictably with (anti-)correlated labels and can be generalized with explicit non-uniform weights, maintaining direct sensitivity to these weights without the hidden prior-dependencies of loss aggregation.}

%% file: deterministic.tex
To better highlight the differences between loss aggregation and label aggregation, we consider the special case of deterministic labels, where $\mathbf{P}(\sY{1} = y_1, \ldots, \sY{K} = y_K) \in \{0, 1\}, \forall x \in \XCal,\, \by \in \{0,1\}^K$. 
Here, we demonstrate that loss aggregation
can exhibit undesirable ``label dictatorship'' behavior,
wherein one label is favored over another.

{We emphasize that the assumption of deterministic labels 
is for illustrative clarity of the ``label dictatorship'' phenomenon (\cref{prop:dictatorship}, \cref{tbl:partial_orders}); our main results on Bayes-optimal scorers (\cref{lemm:bayes-opt-auc-multi-dist} and \cref{lemm:bayes-opt-auc-label-agg}) hold for general $\eta^{(k)}(x)$, as validated in \S\ref{sec:experiments}.}

\subsection{Loss aggregation can yield label dictatorship}
\label{sec:dictatorship}

We first show that the optimal scorer for loss aggregation exhibits a certain \emph{dictatorial} behavior under the deterministic label setting. 
This is easy to see when $K=2$.

\begin{proposition}
\label{prop:dictatorship}
Suppose $\eta^{(k)}(x) \in \{0, 1\}, \forall x \in \XCal$ and $K=2$. 
For mixing weights $a_1, a_2 \geq 0,$ denote $\alpha^{(k)} = \frac{ a_k }{\pi^{(k)} \cdot (1 - \pi^{(k)})}, k \in \{1, 2\}.$
Then 
any Bayes-optimal scorer $f^*$ for the loss aggregated AUC 
in \eqref{eqn:auc-multi-dist} 
satisfies: 
\begin{align*}
  \text{If}~~\alpha^{(1)} > \alpha^{(2)}:\\
        \eta^{(1)}( x ) &> \eta^{(1)}( x' ) \implies f^*( x ) > f^*( x' );\\
    \text{If}~~\alpha^{(2)} > \alpha^{(1)}:\\
        \eta^{(2)}( x ) &> \eta^{(2)}( x' ) \implies f^*( x ) > f^*( x' ).
\end{align*}
\end{proposition}

When $\alpha^{(1)} > \alpha^{(2)}$, the ranking is predominantly determined by $\eta^{(1)}$, and when $\alpha^{(1)} < \alpha^{(2)}$, the ranking is predominantly determined by $\eta^{(2)}$;
i.e.,
\emph{one of the labels acts as a ``dictator''}
in determining the optimal solution.
{This ``label dictatorship'' behavior
can be undesirable in practice. 
Even if both labels provide perfectly clean conditional signals (i.e., $\eta^{(k)}(x) \in \{0,1\}$), a label might dominate the ranking due to its marginal skewness, rather than its contextual importance. 
An example illustrating such an undesirable scenario in information retrieval is provided in Appendix~\ref{app:ir_dictatorship_example}.}
In contrast, label aggregation results in scorers that equally balance both labels even under uniform costs $c_{\bar{y}\bar{y}'} = 1$.

\begin{proposition}
\label{prop:label-agg-deterministic}
Suppose $\eta^{(k)}(x) \in \{0, 1\}, \forall x \in \XCal$ and $K=2$. Then for both costs 
$c_{\bar{y} \bar{y}'}=1$ and 
$c_{\bar{y} \bar{y}'} = 1( \bar{y} > \bar{y}'  ) \cdot | \bar{y} - \bar{y}' |$,
any Bayes-optimal scorer $f^*$ for the label-aggregated AUC 
in \eqref{eqn:auc-label-agg}
satisfies
\begin{align}
        \gamma( x ) &> \gamma( x' ) \implies f^*( x ) > f^*( x' ) 
        \nonumber
        \\
        \gamma( x ) &\defEq 
        \eta^{(1)}( x ) + \eta^{(2)}( x );
\end{align}
or equally, 
$\gamma$
is some non-decreasing transformation of
$f^*$. 
\end{proposition}

\subsection{Illustration of label dictatorship}\label{sec:illustration-label-dictatorship}

We illustrate the above in Figure~\ref{tbl:partial_orders} by showing the partial orders induced by the two objectives for two deterministic binary labels $(\sY{1}, \sY{2})$. 
Each circle denotes an instance with a combination of 
$(\sY{1}, \sY{2})$, 
and the arrows indicate that a particular combination is ranked below another. 
Based on the value of $\alpha^{(1)}$ and $\alpha^{(2)}$, 
loss aggregation either \emph{always} ranks the label combination $(1, 0)$ above $(0, 1)$, 
or always ranks $(1, 0)$ below $(0, 1)$. 
By contrast, label aggregation does not induce any ordering between $(1, 0)$ and $(0, 1)$.

To make this point concrete, recall our MS MARCO example (Section~\ref{sec:pareto-suffice}), where $\sY{1}$ represents \texttt{relevance} and $\sY{2}$ represents \texttt{engagement}. 
Then, node $(1,0)$ conceptually corresponds to the document with high relevance but low engagement (the rightmost column in Table~\ref{table:example_msmarco1}), while $(0,1)$ corresponds to the document with low relevance but high engagement (the middle column in Table~\ref{table:example_msmarco1}).

The ``label dictatorship'' behavior of loss aggregation
---
highlighted by the red arrows in \cref{fig:dictator1,fig:dictator2}
---
translates to imposing a strict, global preference between high-relevance \& low-engagement and low-relevance \&  high-engagement documents.
{
This preference is determined by the relative values of 
$\alpha^{(1)} = a_1 / (\pi^{(1)}(1-\pi^{(1)}))$ 
and $\alpha^{(2)} = a_2 / (\pi^{(2)}(1-\pi^{(2)}))$
(cf.~\cref{prop:dictatorship}). 
Specifically, 
the ranking will favor relevance if $\alpha^{(1)} > \alpha^{(2)}$, and engagement if the inequality is reversed. 
These $\alpha$ values depend on the mixing weights $a_k$ and the marginal label skews $\pi^{(k)}$, regardless of the specific query or document.}
As argued in \S\ref{sec:pareto-suffice}, 
such a fixed global trade-off might be 
undesirable. 
By contrast, label aggregation 
avoids imposing a strict order between high-relevance \& low-engagement and low-relevance \& high-engagement documents.

\begin{figure}[!t]
\centering
    \begin{subfigure}[\small{LoA}\scriptsize{, $\alpha^{(1)} > \alpha^{(2)}$.}]%
        { \label{fig:dictator1}\includegraphics[width=0.16\textwidth]{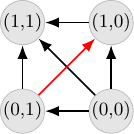}}
    \end{subfigure}
    ~
    \begin{subfigure}[\small{LoA}\scriptsize{, $\alpha^{(1)} < \alpha^{(2)}$.}]%
        {\label{fig:dictator2}\includegraphics[width=0.16\textwidth]{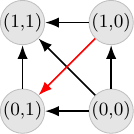}}
    \end{subfigure}
    ~
    \begin{subfigure}[\small{LaA, sum.}]%
        {\label{fig:dictator3}\includegraphics[width=0.16\textwidth]{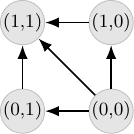}}
    \end{subfigure}
    ~
    \begin{subfigure}[\small{LaA, product.}]%
        {\includegraphics[width=0.16\textwidth]{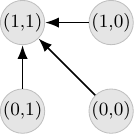}}
    \end{subfigure}
    
     \caption{Illustration of partial orders among examples 
    for two deterministic binary labels 
    induced by:
     (a) Loss aggregation with $\alpha^{(1)} > \alpha^{(2)}$; 
     (b) Loss aggregation with $\alpha^{(1)} < \alpha^{(2)}$; 
     (c) Label Aggregation with $\bar{\mathsf{Y}} = \sY{1} + \sY{2}$; 
     (d) Label Aggregation with $\bar{\mathsf{Y}} = \sY{1} \cdot \sY{2}$.
     As before, $\alpha^{(k)} = \displaystyle\frac{ a_k }{\pi^{(k)} \cdot (1 - \pi^{(k)})}$. %
     Each circle denotes an example with values for the two binary labels. An arrow from a circle marked $(\sY{1}, \sY{2})$ to a circle marked $(\widetilde{\mathsf{Y}}^{(1)}, \widetilde{\mathsf{Y}}^{(2)})$ indicates that the label combination $(\sY{1}, \sY{2})$ is ranked below $(\widetilde{\mathsf{Y}}^{(1)}, \widetilde{\mathsf{Y}}^{(2)})$. The red arrow for the loss aggregation methods depicts the \emph{dictatorial} behavior described in Section \ref{sec:dictatorship}. %
     }
    \label{tbl:partial_orders}
\end{figure}

Before closing, 
in Figure \ref{tbl:partial_orders} 
we contrast label aggregation via 
product $\bar{\mathsf{Y}} = \sY{1} \cdot \sY{2}$ versus summation. 
The product aggregation mechanism is more conservative, in that it  induces only a subset of the partial orders among label combinations induced by sum aggregation. 
Formally:
\begin{lemma}
\label{lem:prod-aggregation}
Suppose $\eta^{(k)}(x) \in \{0, 1\}, \forall x \in \XCal$. 
Let $f^*_{\rm sum}$ and $f^*_{\rm prod}$ be Bayes-optimal scorers for the label aggregated AUC in \eqref{eqn:auc-label-agg} with 
$\bar{\mathsf{Y}}_{\rm sum} = \sum_{k\in[K]} \sY{k}$ and 
$\bar{\mathsf{Y}}_{\rm prod} = \prod_{k\in[K]} \sY{k}$ respectively. Then for any $x, x' \in \XCal$:
\begin{enumerate}[label=(\alph*),itemsep=0pt,topsep=0pt,leftmargin=16pt]
\item 
$\displaystyle
\gamma_{\rm sum}(x) >  \gamma_{\rm sum}(x')  ~\implies~
f^*_{\rm sum}(x) > f^*_{\rm sum}(x');
$
\item 
$\displaystyle
\gamma_{\rm prod}(x) >  \gamma_{\rm prod}(x')  ~\implies~
f^*_{\rm prod}(x) > f^*_{\rm prod}(x');
$
\item 
$\displaystyle
  \gamma_{\rm prod}(x) >  \gamma_{\rm prod}(x')
  ~\implies~ \gamma_{\rm sum}(x) >  \gamma_{\rm sum}(x'),
$
\end{enumerate}
where 
$\gamma_{\rm sum}(x) \defEq \sum_{k \in [K]} \eta^{(k)}(x)$,
and
$\gamma_{\rm prod}(x) \defEq \P(\sY{1}=1,\ldots, \sY{K}=1 \,|\, \mathsf{X}=x)$.
\end{lemma}

%% file: experiments.tex
We present a suite of synthetic and real-world experiments to empirically validate our theoretical findings. 
Specifically, we aim to verify: 
(i) the ``label dictatorship'' phenomenon predicted for loss aggregation (\cref{prop:dictatorship} and \S \ref{sec:dictatorship}), particularly its emergence due to the aforementioned sensitivity to $\pi^{(k)}$; and 
(ii) the consequent tendency for label aggregation (\cref{prop:label-agg-deterministic} and \S\ref{sec:illustration-label-dictatorship}) to offer more balanced handling of multiple labels compared to 
loss aggregation.

To measure the degree of balanced handling of labels, we report the difference between the per-label AUCs:
high differences indicate that one of the two labels is disproportionately favored.
Since this metric can be trivially maximized by a scorer achieving $0.5$ AUC for both labels, we additionally report the minimum over the per-label AUCs.
Note that worst-class performance is a common metric in prior works on fairness~\cite{williamson2019fairnessriskmeasures,sagawa2019distributionally}.
While our focus is on label and loss aggregation, in future work it will be of interest to study how well directly maximizing the minimum per-label AUC would perform.

We present experiments on a synthetic and $3$ real world datasets with training neural models on loss and label aggregation objectives.
When optimizing the AUC-based objectives, following prior works we employ a logistic or hinge surrogate function. 
See Appendix~\ref{s:auc_optimization} for details.

Additionally, in Appendix~\ref{s:synth_details_brute_force} we empirically verify the derived forms of the Bayes-optimal scorers for loss 
and label
aggregation (\cref{lemm:bayes-opt-auc-multi-dist}, \cref{lemm:bayes-opt-auc-label-agg}).

\textbf{Bayes-optimal solutions on synthetic data} %
We begin by analyzing the ``label dictatorship'' phenomenon 
directly on the Bayes-optimal scorers (i.e., \emph{without} any model training) on synthetic data.
We consider a 2D prediction task with two labels, and compare loss and label aggregation as we vary the label skewness $\pi^{(1)}, \pi^{(2)}$.
{This variation in skewness 
directly probes two contrasting theoretical predictions:
recall that
the Bayes-optimal scorer for loss aggregation
(\cref{lemm:bayes-opt-auc-multi-dist})
involves
weighting of individual class-probabilities $\eta^{(k)}(x)$ 
inversely proportional to $\pi^{(k)}\cdot(1-\pi^{(k)})$,
thus
leading 
to ``label dictatorship'' (\cref{prop:dictatorship}). 
By contrast, 
the optimal scorer under
label aggregation 
(with costs $c_{\bar{y}\bar{y}^{\prime}}=|\bar{y}-\bar{y}^{\prime}|$, \cref{lemm:bayes-opt-auc-label-agg}) 
does not involve $\pi^{(k)}$.
}

We generate instances $x \in \mathbb{R}^2$ uniformly from $[-1,1]^2$. We model the two class probability distributions using a logistic model, 
with $\eta^{(1)}(x) = \sigma(\tau \cdot w_1^\top x)$ and $\eta^{(2)}(x) = \sigma(\tau \cdot (w_2^\top x - \rho))$, 
for sigmoid function $\sigma(z) \defEq \frac{1}{1+\exp(-z)}$, 
$w_1 = \big[\frac{1}{\sqrt{2}}, \frac{1}{\sqrt{2}}\big]^\top,$ $w_2 = \left[0, 1\right]^\top$, 
scale parameter $\tau {>0}$, 
and shift parameter $\rho {\geq 0}$. 
As $\tau \to +\infty$, the label distributions become more deterministic. 
As $\rho \to +\infty$, the more skewed is the label distribution $\eta^{(2)}(x)$. 
Specifically, our choice of $w_1$ ensures that $\pi^{(1)} = \mathbb{E}[\eta^{(1)}(\mathsf{X})] = 0.5$; 
we vary the skewness of  
$\pi^{(2)} = \mathbb{E}[\eta^{(2)}(\mathsf{X})]$ alone by varying $\rho$.  

For each instantiation of the above distributions $D^{(1)}$ and $D^{(2)}$, we compute the Bayes-optimal scorers for both loss aggregation \eqref{eqn:bayes-opt-ls} with $a_1=a_2=1$, and label aggregation \eqref{eq:bayes-opt-la}, and evaluate the difference $\Delta_{\rm AUC} = |\text{AUC}(f; D^{(1)}) - \text{AUC}(f; D^{(2)})|$ on $10^5$ samples drawn from $(D^{(1)}, D^{(2)})$. In Figure \ref{fig:pi2_sweep}, we plot the differences in the per-label AUC as a function of the second label's skewness $\pi^{(2)}$.  Observe that for most skewness values, the optimal scorer for loss aggregation leads to a larger gap in AUCs between the two labels. This is due to the optimal scorer for loss aggregation \eqref{eqn:bayes-opt-ls} closely depending on the label priors $\pi^{(1)}$ and $\pi^{(2)}$. In fact, the closer the distributions are to being deterministic, the larger is $\Delta_{\rm AUC}$ for loss aggregation. This observation closely aligns with the dictatorial behavior described in Section \ref{sec:dictatorship} for loss aggregation under deterministic labels.

\begin{figure}[!t]
    \centering
    
    {\includegraphics[width=1.0\columnwidth]{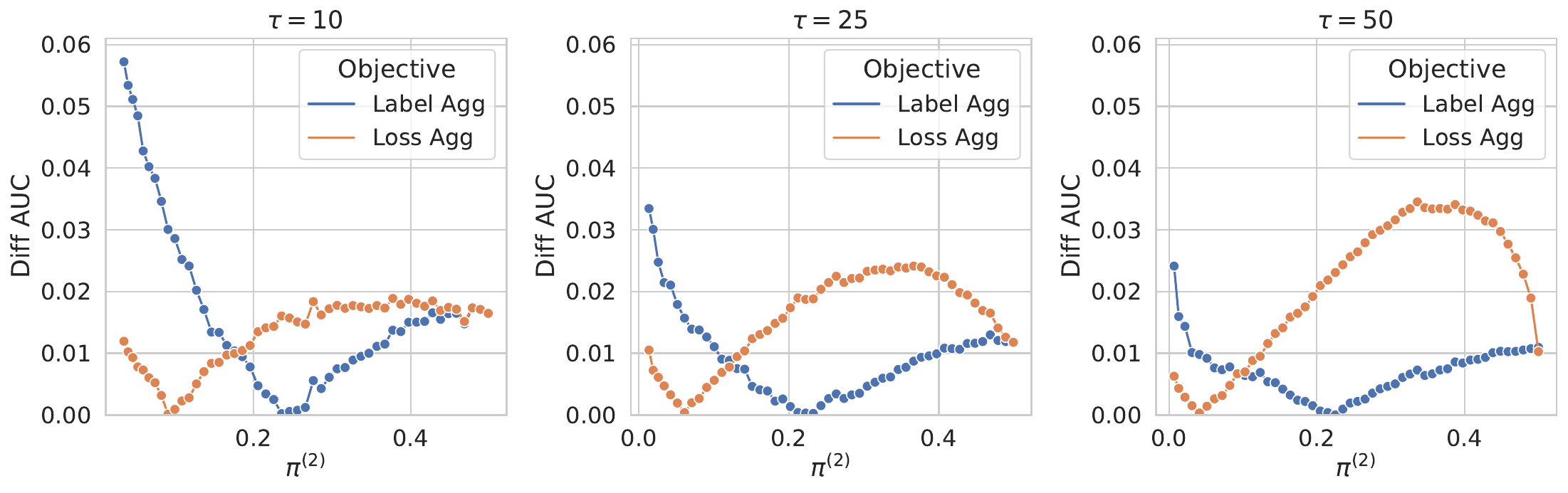}}
    \vspace{-0.3cm}
    \caption{Plot of $|\text{AUC}(\cdot; D^{(1)}) - \text{AUC}(\cdot; D^{(2)})|$ for optimal scorers  as a function of skewness in label $\sY{2}$. 
    We compare label aggregation and loss aggregation on the synthetic dataset as described in the main text. 
    We fix $\pi^{(1)} = \mathbf{P}(\mathsf{Y}^{(1)}=1) = 0.5$ and vary $\pi^{(2)} = \mathbf{P}(\mathsf{Y}^{(2)}=1)$.
    Larger values of sigmoid scaling parameter $\tau$ make the label distribution closer to deterministic.
    A lower difference indicates a more balanced label treatment by the optimal solution; 
    loss aggregation leads to a higher difference.
    Figure~\ref{fig:pi2_sweep_appendix} reports individual per-label AUC metrics.\label{fig:pi2_sweep}}
\end{figure}

\textbf{\texttt{Banking}.}\ 
We consider the UCI \texttt{Banking} dataset composed of information about Bank customers, 
advertising campaign details, 
and the success thereof~\citep{MORO201422}.
We take the 
\emph{mortgage} and \emph{loan} variables as two targets.

To study the behavior of label and loss aggregation under label skews, 
{we re-sample this dataset with varying proportions of the positive mortgage label, and average results over 25 trials.   %
We train a linear model on numerical features using Adam for $100$ epochs.
In Table~\ref{tbl:bank_res} and Figure \ref{fig:banking-sweep-pi1}, we compare the scorers learned by
optimizing the 
loss aggregation objective \eqref{eqn:auc-multi-dist} with $a_1=a_2=1$ and the 
label aggregation objective \eqref{eqn:auc-label-agg} with costs $c_{\bar{y} \bar{y}'} = \1( \bar{y} > \bar{y}'  ) \cdot | \bar{y} - \bar{y}' |$. 
We employ the logistic surrogate loss in both cases.
Notice that
label aggregation is able to better balance between the two labels, yielding the highest Min AUC metric.}

\begin{table}[!h]
    \centering
    \ifarxivversion

        \begin{tabular}{@{}lcccc@{}}
            \toprule
            \toprule
            \textbf{Objective} & \textbf{AUC mortgage} $\uparrow$ & \textbf{AUC loan} $\uparrow$
            & \textbf{Diff AUC} $\downarrow$ & \textbf{Min AUC} $\uparrow$
            \\ %
            \toprule
                AUC(mortgage) & \best{0.637 $\pm$ 0.004}  & 0.523 $\pm$ 0.002  & 0.113 $\pm$ 0.004  & 0.523 $\pm$ 0.002\\
                \midrule
                
                 AUC(loan) & 0.550 $\pm$ 0.005  & \best{0.573 $\pm$ 0.002}  & \best{0.023 $\pm$ 0.005}  & 0.550 $\pm$ 0.005 \\
                \midrule
                
                AUC$_{\rm LaA}$  & 0.616 $\pm$ 0.003  & 0.562 $\pm$ 0.002  & 0.054 $\pm$ 0.005  & \best{0.562 $\pm$ 0.002}  \\
                
                \midrule
                AUC$_{\rm LoA}^{(1, 1)}$  & 0.626 $\pm$ 0.003  & 0.555 $\pm$ 0.002  & 0.071 $\pm$ 0.005  & 0.555 $\pm$ 0.002\\
                \bottomrule
        \end{tabular}%
    \else
        \renewcommand{\arraystretch}{1}
        
        \resizebox{1.0\linewidth}{!}{%

        \begin{tabular}{@{}lcccc@{}}
            \toprule
            \toprule
            \textbf{Objective} & \textbf{AUC mortgage} $\uparrow$ & \textbf{AUC loan} $\uparrow$
            & \textbf{Diff AUC} $\downarrow$ & \textbf{Min AUC} $\uparrow$
            \\ %
            \toprule
                AUC(mortgage) & \best{0.637 $\pm$ 0.004}  & 0.523 $\pm$ 0.002  & 0.113 $\pm$ 0.004  & 0.523 $\pm$ 0.002\\
                \midrule
                
                 AUC(loan) & 0.550 $\pm$ 0.005  & \best{0.573 $\pm$ 0.002}  & \best{0.023 $\pm$ 0.005}  & 0.550 $\pm$ 0.005 \\
                \midrule
                
                AUC$_{\rm LaA}$  & 0.616 $\pm$ 0.003  & 0.562 $\pm$ 0.002  & 0.054 $\pm$ 0.005  & \best{0.562 $\pm$ 0.002}  \\
                
                \midrule
                AUC$_{\rm LoA}^{(1, 1)}$  & 0.626 $\pm$ 0.003  & 0.555 $\pm$ 0.002  & 0.071 $\pm$ 0.005  & 0.555 $\pm$ 0.002\\
                \bottomrule
        \end{tabular}%
        }
    \fi
    \caption{Results on the \texttt{Banking} dataset with the proportion of positive mortgage label set to 0.9. While optimizing an objective for an individual label maximizes the corresponding AUC, we find label aggregation to strike a balance between the two evaluation metrics, %
     yielding the \textit{highest Min AUC}. Here, AUC$_{\rm LaA}$ = AUC(mortgage + loan) and AUC$_{\rm LoA}^{(1, 1)}$ = AUC(mortgage) + AUC(loan).  
    }
    \label{tbl:bank_res}
\end{table}

\begin{figure}[!t]
    \centering
    
    {\includegraphics[width=1.0\columnwidth]{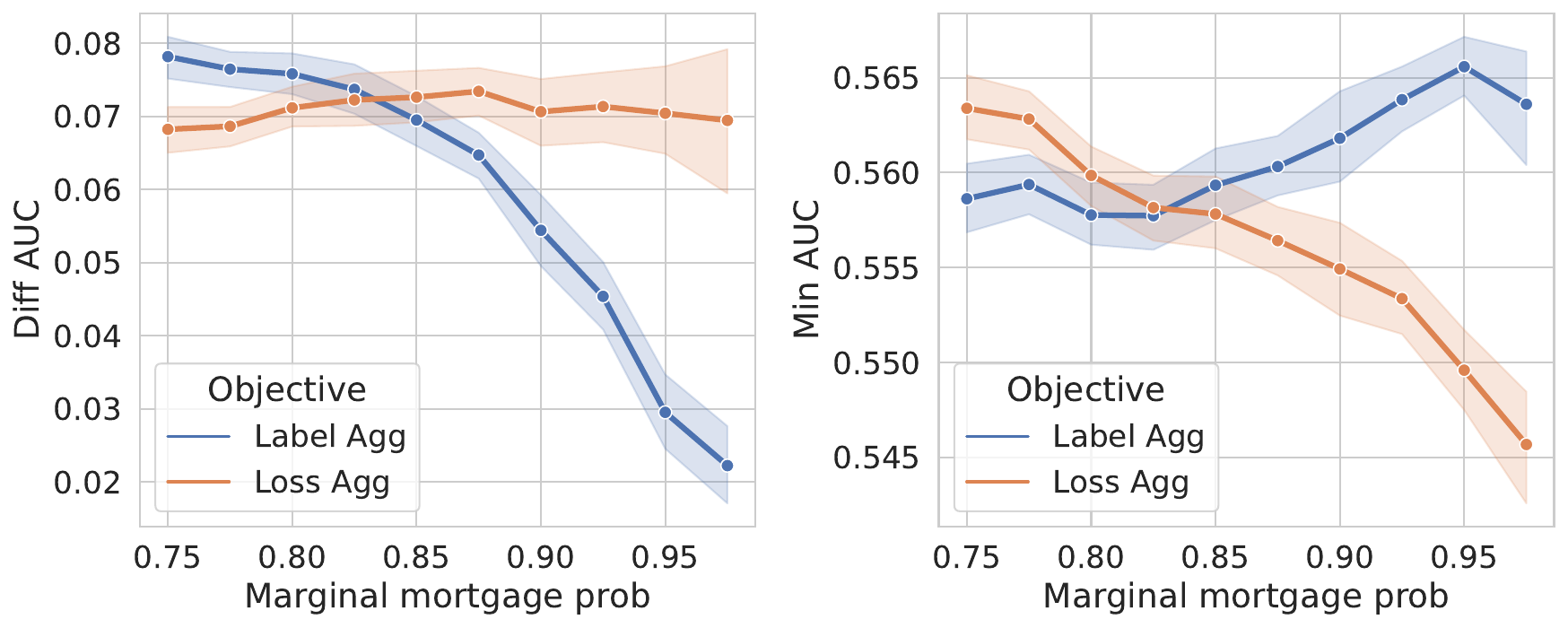}}
    \vspace{-0.3cm}
    \caption{Plots of $|\text{AUC(mortgage)} - \text{AUC(loan)}|$ (\textit{lower} is better) and $\min\{\text{AUC(mortgage)}, \text{AUC(loan)}\}$ (\textit{higher} is better)  on the \texttt{Banking} dataset. We compare label aggregation and loss aggregation as we vary the marginal probability of the mortgage label. We find label aggregation to fare better in the higher skewness regime.}
    \label{fig:banking-sweep-pi1}
\end{figure}

\begin{figure}[!t]
    \centering
    
    {\includegraphics[width=1.0\columnwidth]{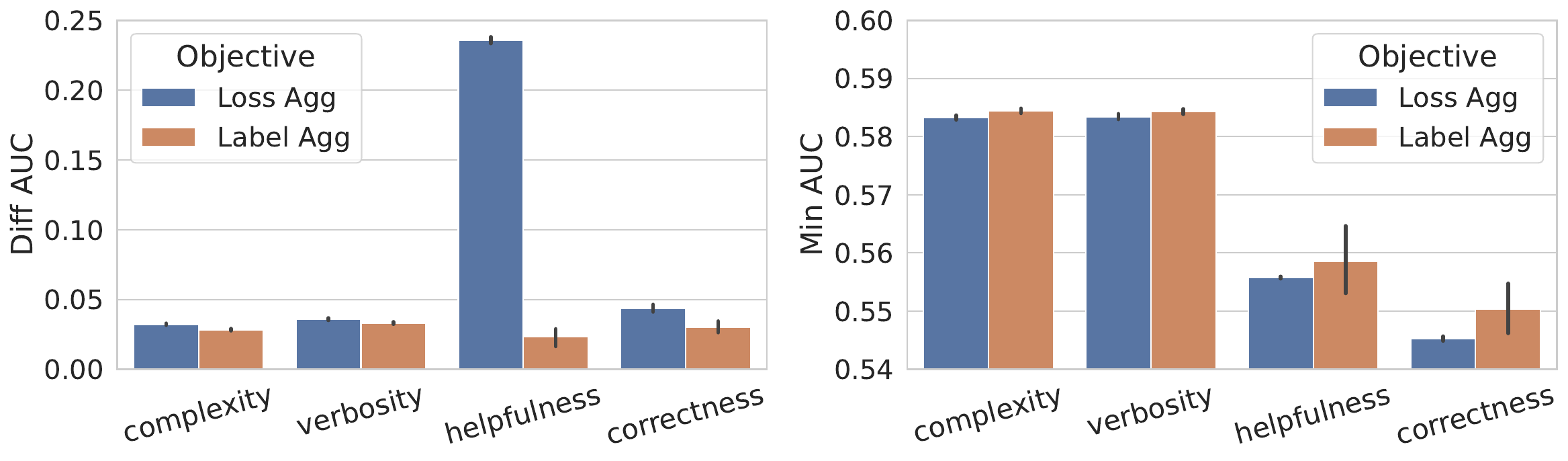}}
    \vspace{-0.3cm}
    \caption{Plot of
    Diff AUC $|\text{AUC}(\cdot; D^{(1)}) - \text{AUC}(\cdot; D^{(2)})|$ (\textit{lower} is better) and 
    Min AUC $\min\{\text{AUC}(\cdot; D^{(1)}), \text{AUC}(\cdot; D^{(2)})\}$ (\textit{higher} is better) 
    for scorers learned with different labels $\sY{2}$ (sorted according to their skewness, from the most skewed to the least skewed). Experiments on the \texttt{HelpSteer} dataset, where the first label is fixed to coherence (which is the closest to balance, and thus, $\pi^{(1)} = \mathbf{P}(\mathsf{Y}^{(1)}=1) \approx 0.5$), and the second label is chosen over the set of all remaining labels (and so, $\pi^{(2)} = \mathbf{P}(\mathsf{Y}^{(2)}=1)$ is of varying skewness). We find loss aggregation to lead to a higher difference between per-label AUC metrics.\label{fig:pi2_sweep_helpsteer}}
\end{figure}

\textbf{\texttt{HelpSteer}.}\
\texttt{HelpSteer}~\cite{wang2023helpsteer} consists of evaluations of LLM responses across 5 categories: 
helpfulness, factuality, correctness, coherence, complexity and verbosity. 
On each category, responses are rated from $0$ to $4$, where higher means better.
While originally intended for usage in LLM alignment, we repurpose the dataset here to illustrate the tradeoffs when ranking along competing signals.

We consider a task of ranking concatenations of a prompt and candidate response.
We binarize each category based on whether or not it equals $4$, 
and fix the first label to coherence (the most balanced category), while iteratively choosing the second label to be each of the remaining categories.
Figure~\ref{fig:pi2_sweep_helpsteer} shows that label aggregation yields the lowest difference between per-label AUC across all settings.

\textbf{\texttt{MSLR}.}\
We next consider \texttt{MSLR Web30k}, a dataset of
users' query-document interactions~\citep{QinL13}. 
We follow the methodology from
\citet{mahapatra2023multi} in constructing the target scoring function by removing Query-URL Click Count (Click), and using it in conjunction with the Relevance label.
We train an MLP model over the concatenated features and follow hyper-parameters specified in the {TensorFlow ranking} library~\citep{TensorflowRankingKDD2019}.

In Table~\ref{tbl:mslr_res}, we report results across different objectives.
Since the AUC over clicks is always seen to be higher than AUC over relevance, $\min(\text{AUC click}, \text{AUC relevance}) = \text{AUC relevance}$ across all objectives.
We find that, 
compared to the label aggregation solution, loss aggregation solutions tend to over-optimize one metric over the other, depending on the weights.
On the other hand, label aggregation yields the highest $\min(\text{AUC click}, \text{AUC relevance})$ and Pareto dominates the loss aggregation with equal weights.

\begin{table}[!t]
    \centering
    \ifarxivversion 
        \begin{tabular}{@{}lcccc@{}}
            \toprule
            \toprule
            \textbf{Objective} & \textbf{AUC Click} $\uparrow$ & \textbf{AUC Rel} $\uparrow$
            & \textbf{Diff AUC} $\downarrow$ & \textbf{Min AUC} $\uparrow$\\
            \toprule
                AUC(Click) &  \best{0.74} & 0.61 & 0.13 & 0.61\\
                \midrule
                AUC(Rel) & 0.70 & 0.67 & \best{0.03} & 0.67 \\
                \midrule
                AUC$_{\rm LaA}$ &  0.73 & \best{0.68} & 0.05 & \best{0.68}\\
                \midrule
                AUC$_{\rm LoA}$  &  0.73 &	0.67 & 0.06 & 0.67\\
                \bottomrule
        \end{tabular}%
    \else
        \renewcommand{\arraystretch}{1}
        \resizebox{\linewidth}{!}{%
        \begin{tabular}{@{}lcccc@{}}
            \toprule
            \toprule
            \textbf{Objective} & \textbf{AUC Click} $\uparrow$ & \textbf{AUC Rel} $\uparrow$
            & \textbf{Diff AUC} $\downarrow$ & \textbf{Min AUC} $\uparrow$\\
            \toprule
                AUC(Click) &  \best{0.74} & 0.61 & 0.13 & 0.61\\
                \midrule
                AUC(Rel) & 0.70 & 0.67 & \best{0.03} & 0.67 \\
                \midrule
                AUC$_{\rm LaA}$ &  0.73 & \best{0.68} & 0.05 & \best{0.68}\\
                \midrule
                AUC$_{\rm LoA}$  &  0.73 &	0.67 & 0.06 & 0.67\\
                \bottomrule
        \end{tabular}%
        }
    \fi
    \caption{Results on the \texttt{MSLR} dataset. Learning on both label (click and relevance) helps even when evaluated on each label alone. %
    Label aggregation strikes a better balance between the two objectives, and achieves the largest minimum AUC over the two labels. %
    }
    \label{tbl:mslr_res}
\end{table}

%% file: related.tex
\textbf{Bipartite ranking}.
Bipartite ranking is a fundamental supervised learning problem,
with intimate ties to binary classification and class-probability estimation~\citep{Narasimhan:2013}.
A large body of work has studied various theoretical aspects of the problem,
including 
statistical generalization~\citep{Agarwal:2005,Agarwal:2014},
consistency of suitable surrogate loss minimisation~\citep{Clemencon:2008,Uematsu:2015,Menon:2016},
relation to classic supervised learning problems~\citep{Cortes:2003,Kotlowski:2011,Narasimhan:2013},
effective algorithm design~\citep{Freund:2003},
and extensions to \emph{top-ranking} settings~\citep{Rudin:2009,Agarwal2011,Li:2014}.

\textbf{Multi-objective ranking}.
There has been much work on
finding Pareto optimal solutions for
multi-objective ranking~\cite{mahapatra2023multi}.
A primary focus has been on modifying the LambdaMART objective~\cite{burges2010ranknet} and evaluating on the nDCG evaluation metric, which the original LambdaMART objective optimizes the upper bound for~\cite{lambdaloss}.
Two prominent lines of works focus on loss aggregation (or linear scalarization)~\citep{ruchte2021scalableparetoapproximationdeep} and label aggregation \cite{Svore2011,Agarwal2011,Dai2011,Carmel2020,wei2023aggregate}. 
To the best of our knowledge, prior works neither considered multi-objective AUC optimization, nor theoretically analyzed the corresponding optimal solutions.

\textbf{Multi-label ranking}.
Multi-label classification involves learning to predict \emph{multiple} labels associated with a given instance~\citep{Tsoumakas:2010,Read:2009,Dembczynski:2010}.
Multi-label \emph{ranking} generalizes this to learn a \emph{ranker} over the multiple candidate labels~\citep{Brinker:2006,Furnkranz:2008,Dembczynski:2012}.
Canonically, such a ranker may be assessed based on an \emph{instance-specific} analogue of the AUC~\citep{Wu:2021}.
While the problem setting is similar to what we consider,
the goal is fundamentally different:
multi-label ranking seeks to produce a score for \emph{every} candidate label,
while we seek to produce a \emph{single} score that synthesizes the labels.

%% file: conclusions.tex
We have studied the problem of bipartite ranking from multiple labels,
with a characterization of the Bayes-optimal solutions for 
loss and label aggregation techniques.
While these optimal scorers have a similar form,
we established that loss aggregation can lead to an undesirable ``dictatorship'' issue.
In future work, it would be of interest to 
study \emph{surrogate} losses for the loss and label aggregation strategies.
Another interesting direction would be to move beyond AUC to metrics such as the nDCG~\citep{Wang:2013}.

\clearpage

\section*{Impact Statement}

This paper presents work goals of which is to advance the field of machine learning, specifically in the area of bipartite ranking with multiple labels. 
While there are many potential societal consequences of our work, we don't feel they must be specifically highlighted here. 
The presented analysis of loss and label aggregation methods provides a deeper understanding of multi-label bipartite ranking, potentially leading to improved algorithms in information retrieval and medical diagnosis applications.

%% file: appendix.tex
\section{Additional helper results}

\begin{lemma}
\label{opt:weight}
Let $\mathscr{S}$ denote the set of pairwise scorers $s \colon \XCal \times \XCal \to \Real$ satisfying the \emph{symmetry} condition: $s( x, x' ) = s( x', x )$.
For any distribution $\mu$ and weight function $w \colon \XCal \times \XCal \to \Real$, let
\begin{equation}
    \label{eqn:pairwise-opt}
    s^* \in \underset{s \in \mathscr{S}}{\argmax} \, \mathbb{E}_{\mathsf{X} \sim \mu} \mathbb{E}_{\mathsf{X}' \sim \mu} \left[ w( \mathsf{X}, \mathsf{X}' ) \cdot H( s( \mathsf{X}, \mathsf{X}' ) )  \right].
\end{equation}
Then,
$$ s^*( x, x' ) > 0 \iff w( x, x' ) - w( x', x ) > 0. $$
\end{lemma}

\begin{proof}[Proof of Lemma~\ref{opt:weight}]
By symmetry of the expectation, and the definition of $H( \cdot )$,
the optimal solution is equivalently
\begin{align*}
    s^* &\in \underset{s \in \mathscr{S}}{\argmax} \, \mathbb{E}_{\mathsf{X} \sim \mu} \mathbb{E}_{\mathsf{X}' \sim \mu} \left[ w( \mathsf{X}, \mathsf{X}' ) \cdot H( s( \mathsf{X}, \mathsf{X}' ) ) + w( \mathsf{X}', \mathsf{X} ) \cdot H( -s( \mathsf{X}, \mathsf{X}' ) )  \right] \\
    &= \underset{s \in \mathscr{S}}{\argmax} \, \mathbb{E}_{\mathsf{X} \sim \mu} \mathbb{E}_{\mathsf{X}' \sim \mu} \left[ 
    \begin{cases} 
        w( \mathsf{X}, \mathsf{X}' ) &\text{ if } s( \mathsf{X}, \mathsf{X}' ) > 0 \\
        w( \mathsf{X}', \mathsf{X} ) &\text{ if } s( \mathsf{X}, \mathsf{X}' ) < 0 \\
        \frac{w( \mathsf{X}, \mathsf{X}' ) + w( \mathsf{X}', \mathsf{X} )}{2} &\text{ if } s( \mathsf{X}, \mathsf{X}' ) = 0 \\
    \end{cases} \right].
\end{align*}
For each fixed $(x, x') \in \XCal \times \XCal$, we thus have
$$ s^*( x, x' ) > 0 \iff w( x, x' ) - w( x', x ) > 0. $$
Note that if $w( x, x' ) = w( x', x )$, then the choice of $s^*( x, x' )$ may be arbitrary.
\end{proof}

\begin{lemma}
\label{lemm:opt-weight-decomposed}
For any distribution $\mu$ and weight function $w \colon \XCal \times \XCal \to \Real$, let
\begin{equation}
    \label{eqn:decomposable-pairwise-opt}
    f^* \in \underset{f \colon \XCal \to \Real}{\argmax} \, \mathbb{E}_{\mathsf{X} \sim \mu} \mathbb{E}_{\mathsf{X}' \sim \mu} \left[ w( \mathsf{X}, \mathsf{X}' ) \cdot H( f( \mathsf{X} ) - f( \mathsf{X}' ) )  \right].
\end{equation}
Then, if there exists some $g \colon \XCal \to \Real$ such that
$$ g( x ) > g( x' ) \iff w( x, x' ) > w( x', x ), $$
we must have
\begin{equation}
    \label{eqn:condition}
    f^*( x ) > f^*( x' ) \iff g( x ) > g( x' ).
\end{equation}
\end{lemma}

\begin{proof}[Proof of Lemma~\ref{lemm:opt-weight-decomposed}]
Let
\begin{align*}
    R_{\rm pair}( s ) &\defEq \mathbb{E}_{\mathsf{X} \sim \mu} \mathbb{E}_{\mathsf{X}' \sim \mu} \left[ w( \mathsf{X}, \mathsf{X}' ) \cdot H( s( \mathsf{X}, \mathsf{X}' ) )  \right] \\
    R_{\rm diff}( f ) &\defEq \mathbb{E}_{\mathsf{X} \sim \mu} \mathbb{E}_{\mathsf{X}' \sim \mu} \left[ w( \mathsf{X}, \mathsf{X}' ) \cdot H( f( \mathsf{X} ) - f( \mathsf{X}' ) )  \right].
\end{align*}
If there exists a $g$ satisfying the prescribed condition, then by Lemma~\ref{opt:weight}, 
any optimal \emph{pairwise} scorer $s^*$ for $R_{\rm pair}$ satisfies 
$s^*( x, x' ) > 0 \iff g( x ) > g( x' )$.

Now let $g_{\rm diff}( x, x' ) \defEq g( x ) - g( x' )$.
Then, $g_{\rm diff}( x, x' ) > 0 \iff s^*( x, x' ) > 0 \iff w( x, x' ) > w( x', x )$,
and so $g_{\rm diff}$ is an optimal pairwise scorer for $R_{\rm pair}$.
Thus, for any alternate scorer $\bar{g} \colon \XCal \to \Real$,
with $\bar{g}_{\rm diff}( x, x' ) \defEq \bar{g}( x ) - \bar{g}( x' )$,
we must have $R_{\rm pair}( g_{\rm diff} ) \leq R_{\rm pair}( \bar{g}_{\rm diff} )$.
This however implies that $R_{\rm diff}( g ) \leq R_{\rm diff}( \bar{g} )$,
i.e.,
$g$ is an optimal scorer for the present objective $R_{\rm diff}$~\eqref{eqn:decomposable-pairwise-opt}.

Now suppose $f^*$ is any candidate optimal solution for $R_{\rm diff}$.
Then, $f^*$ must satisfy the given condition~\eqref{eqn:condition}, or else
(again appealing to Lemma~\ref{opt:weight})
$f^*_{\rm diff}$ will not be optimal for $R_{\rm pair}$.
Thus, the result follows.
\end{proof}

\begin{lemma}[\citet{Agarwal:2005}]
\label{lemm:auc-bipartite-equivalence}
For a scorer $f \colon \XCal \to \Real$ and distribution $D$, the \AUCROC{} is equally
\begin{align}
    \label{eqn:auc-bp}
    {\rm AUC}( f; D ) &\defEq \mathbb{E}_{\mathsf{X} \sim q_{+}} \mathbb{E}_{\mathsf{X}' \sim q_{-}} \left[ H( f( \mathsf{X} ) - f( \mathsf{X}' ) ) \right] \\
    \nonumber
    H( z ) &\defEq 1( z > 0 ) + \frac{1}{2} \cdot 1( z = 0 ).
\end{align}
where
$q_+(x) = \P(x|y=1)$ and $q_-(x) = \P(x|y=0)$ denote the class-conditional distributions.
\end{lemma}

\begin{proof}[Proof of Lemma~\ref{lemm:auc-bipartite-equivalence}]
See proof of Lemma~\ref{lemm:auc-multi}, which is a strict generalisation.
\end{proof}

\begin{lemma}[\citet{Agarwal:2014}]
\label{lemm:auc-bipartite-v2}
For any scorer $f \colon \XCal \to \Real$ and distribution $D$, the AUC-ROC is equally
\begin{align}
{\rm AUC}( f; D ) &=  \mathbb{E}_{\mathsf{X} \sim \mu} \mathbb{E}_{\mathsf{X}' \sim \mu}\left[ w( \mathsf{X}, \mathsf{X'} ) \cdot H( f( \mathsf{X} ) - f( \mathsf{X}' ) ) \right] 
\label{eqn:auc-bipartite-weighted}
\\
w( x, x' ) &\defEq \frac{1}{\pi(1-\pi)} \cdot \eta( x ) \cdot (1 - \eta( x') ).
\nonumber
\end{align}
\end{lemma}

\begin{proof}[Proof of Lemma~\ref{lemm:auc-bipartite-v2}]
By an application of Bayes' rule to Lemma~\ref{lemm:auc-bipartite-equivalence},
\begin{align*}
{\rm AUC}( f; \Dbp ) &= \mathbb{E}_{\mathsf{X} \sim q_{+}} \mathbb{E}_{\mathsf{X}' \sim q_{-}} \left[ H( f( \mathsf{X} ) - f( \mathsf{X}' ) ) \right] \\
&= \frac{1}{\pi \cdot (1 - \pi)} \cdot \mathbb{E}_{\mathsf{X} \sim \mu} \mathbb{E}_{\mathsf{X}' \sim \mu} \left[ \eta( \mathsf{X} ) \cdot (1 - \eta( \mathsf{X}') ) \cdot H( f( \mathsf{X} ) - f( \mathsf{X}' ) ) \right].
\end{align*}
\end{proof}

\begin{lemma}[\citet{Uematsu:2015}]
\label{lemm:auc-multi}
For any scorer $f \colon \XCal \to \Real$, distribution $D_{\rm mp}$, and costs $\{ c_{y y'} \}$,
\begin{align*}
{\rm AUC}( f; D_{\rm mp} ) &= \mathbb{E}_{\mathsf{X} \sim \mu} \mathbb{E}_{\mathsf{X}' \sim \mu} \left[ w( \mathsf{X}, \mathsf{X}' ) \cdot H( f( \mathsf{X} ) - f( \mathsf{X}' ) ) \right] \\
w( x, x' ) &\defEq \frac{1}{v} \cdot \sum_{y \in \YCal} \sum_{y' \in \YCal} 1( y > y' ) \cdot c_{y y'} \cdot \eta_y( x ) \cdot \eta_{y'}( x' ) \\
v &\defEq \sum_{y \in \YCal} \sum_{y' \in \YCal} 1( y > y' ) \cdot c_{y y'} \cdot \pi_y \cdot \pi_{y'}.
\end{align*}
\end{lemma}
\begin{proof}[Proof of Lemma~\ref{lemm:auc-multi}]
We provide a proof for completeness.
By definition of conditional expectation,\footnote{We elide again details on the choice of base measure used to define the density on $\XCal$.}
\begin{align*}
& {\rm AUC}( f; D_{\rm mp} )\\
={}& \mathbb{E}_{\mathsf{X}} \mathbb{E}_{\mathsf{X}'} \left[ c_{\mathsf{Y} \mathsf{Y'}} \cdot H( f( \mathsf{X} ) - f( \mathsf{X}' ) ) \mid \mathsf{Y} > \mathsf{Y}' \right] \\
={}& \int_{\XCal \times \XCal} \sum_{y, y'} \left[ D( x, x', y, y' \mid \mathsf{Y} > \mathsf{Y}' ) \cdot c_{{y} {y'}} \cdot H( f( x ) - f( x' ) ) \right] \, \mathrm{d}x \, \mathrm{d}x' \\
={}& \int_{\XCal \times \XCal} \sum_{y, y'} \left[ \frac{1( y > y' ) \cdot D( x, x', y, y' )}{D( \mathsf{Y} > \mathsf{Y}' )} \cdot c_{y y'} \cdot  H( f( x ) - f( x' ) ) \right] \, \mathrm{d}x \, \mathrm{d}x' \\
={}& \frac{1}{D( \mathsf{Y} > \mathsf{Y}' )} \cdot \int_{\XCal \times \XCal} \sum_{y, y'} \left[ {1( y > y' ) \cdot D( x, x' ) \cdot D( y, y' \mid x, x' )} \cdot c_{y y'} \cdot  H( f( x ) - f( x' ) ) \right] \, \mathrm{d}x \, \mathrm{d}x' \\
={}& \frac{1}{D( \mathsf{Y} > \mathsf{Y}' )} \cdot \mathbb{E}_{\mathsf{X} \sim \mu} \mathbb{E}_{\mathsf{X}' \sim \mu} \left[ \left[ \sum_{y, y'} {1( y > y' ) \cdot \eta_y( x ) \cdot \eta_{y'} ( x' )} \cdot c_{y y'} \right] \cdot H( f( x ) - f( x' ) ) \right].
\end{align*}
Now observe that
\begin{align*}
    D( \mathsf{Y} > \mathsf{Y}' ) &= \sum_{y \in [L]} \sum_{y' \in [L]} 1( y > y' ) \cdot \pi_y \cdot \pi_{y'} \\
    &= \frac{1}{2} \cdot \left[ 1 - \sum_{y \in \YCal} \pi^2_y \right].
\end{align*}
\end{proof}

\begin{lemma}
\label{lemm:auc-multi-dist}
For any scorer $f \colon \XCal \to \Real$ and distributions $\{ D^{(k)} \}_{k \in [K]}$, 
the loss aggregated AUC is
\begin{align*}
    {\rm AUC}_{\rm LoA}( f ) &= 
    \mathbb{E}_{\mathsf{X} \sim \mu} \mathbb{E}_{\mathsf{X}' \sim \mu} \left[ w( \mathsf{X}, \mathsf{X'} ) \cdot H( f( \mathsf{X} ) - f( \mathsf{X}' ) ) \right] \\
    w( x, x' ) &\defEq \frac{1}{K} \sum_{k \in [K]} \frac{\eta^{(k)}( x ) \cdot (1 - \eta^{(k)}( x') )}{\pi^{(k)} \cdot (1 - \pi^{(k)})}.
\end{align*}
\end{lemma}

\begin{proof}[Proof of Lemma~\ref{lemm:auc-multi-dist}]
Building on Definition~\ref{lemm:auc-bipartite}, we have
\begin{align*}
   & {\rm AUC}( f; \{ D^{(k)}_{\rm bp} \}_{k \in [K]} )\\
    ={}& \frac{1}{K} \sum_{k \in [K]} \mathbb{E}_{\mathsf{X}} \mathbb{E}_{\mathsf{X}'} \left[ H( f( \mathsf{X} ) - f( \mathsf{X}' ) ) \mid \mathsf{Y}^{(k)} > \mathsf{Y}'^{(k)} \right] \\
    ={}& \frac{1}{K} \sum_{k \in [K]} \frac{1}{\pi^{(k)} \cdot (1 - \pi^{(k)})} \cdot \mathbb{E}_{\mathsf{X} \sim \mu} \mathbb{E}_{\mathsf{X}' \sim \mu} \left[ \eta^{(k)}( \mathsf{X} ) \cdot (1 - \eta^{(k)}( \mathsf{X}') ) \cdot H( f( \mathsf{X} ) - f( \mathsf{X}' ) ) \right].
\end{align*}
The result thus follows.
\end{proof}

\section{Proofs of results in body}\label{sec:proofs}

\begin{proof}[Proof of Lemma~\ref{lemm:auc-bipartite-opt}]
First, observe that by symmetry
\begin{align*}
{\rm AUC}( f; \Dbp ) ={}& \frac{1}{\pi \cdot (1 - \pi)} \cdot \mathbb{E}_{\mathsf{X} \sim \mu} \mathbb{E}_{\mathsf{X}' \sim \mu} \left[ \eta( \mathsf{X} ) \cdot (1 - \eta( \mathsf{X}') ) \cdot H( f( \mathsf{X} ) - f( \mathsf{X}' ) ) \right] \\
={}& \frac{1}{\pi \cdot (1 - \pi)} \cdot \mathbb{E}_{\mathsf{X} \sim \mu} \mathbb{E}_{\mathsf{X}' \sim \mu} \left[ \frac{1}{2} \cdot \left[ \eta( \mathsf{X} ) \cdot (1 - \eta( \mathsf{X}') ) \cdot H( f( \mathsf{X} ) - f( \mathsf{X}' ) ) +\right.\right.\\
& \left.\left. \eta( \mathsf{X}' ) \cdot (1 - \eta( \mathsf{X}) ) \cdot H( f( \mathsf{X}' ) - f( \mathsf{X} ) ) \right] \right] \\
={}& \mathbb{E}_{\mathsf{X} \sim \mu} \mathbb{E}_{\mathsf{X}' \sim \mu} \left[ \frac{1}{2} \cdot \left[ w( \mathsf{X}, \mathsf{X}' ) \cdot H( f( \mathsf{X} ) - f( \mathsf{X}' ) ) + w( \mathsf{X}', \mathsf{X} )) \cdot H( f( \mathsf{X}' ) - f( \mathsf{X} ) ) \right] \right].
\end{align*}
Now observe that
$$ w( x, x' ) > w( x', x ) \iff \eta( x ) > \eta( x' ). $$
Thus, by Lemma~\ref{lemm:opt-weight-decomposed}, the optimal scorer satisfies $f^*( x ) - f^*( x' ) > 0 \iff \eta( x ) - \eta( x' ) > 0$.
\end{proof}

\begin{proof}[Proof of Lemma~\ref{lemm:auc-multipartite-opt}]
We provide a proof for $L = 3$ for completeness.
For the general case, see~\citet{Uematsu:2015}.

Observe that
\begin{align*}
   &  w( x, x' ) - w( x', x )\\
    ={}& \frac{1}{v} \cdot \left[ \sum_{y \in \YCal} \sum_{y' \in \YCal} 1( y > y' ) \cdot c_{yy'} \cdot \eta_{y}( x ) \cdot \eta_{y'}( x' ) - \sum_{y \in \YCal} \sum_{y' \in \YCal} 1( y > y' ) \cdot c_{yy'} \cdot \eta_{y}( x' ) \cdot \eta_{y'}( x ) \right] \\
    ={}& \frac{1}{v} \cdot \left[ \sum_{y \in \YCal} \sum_{y' \in \YCal} 1( y > y' ) \cdot c_{yy'} \cdot \eta_{y}( x ) \cdot \eta_{y'}( x' ) - \sum_{y \in \YCal} \sum_{y' \in \YCal} 1( y < y' ) \cdot c_{y'y} \cdot \eta_{y}( x ) \cdot \eta_{y'}( x' ) \right] \\
    ={}& \frac{1}{v} \cdot \left[ \sum_{y \in \YCal} \sum_{y' \in \YCal} \alpha_{y y'} \cdot \eta_{y}( x ) \cdot \eta_{y'}( x' ) \right] \\
    ={}& \frac{1}{v} \cdot \left[ \sum_{y \in \YCal} \beta_{y}( x' ) \cdot \eta_{y}( x ) \right],
\end{align*}
where $\alpha_{y y'} \defEq 1( y > y' ) \cdot c_{y y'} - 1( y < y' ) \cdot c_{y' y}$,
and $\beta_{y}( x' ) \defEq \sum_{y' \in \YCal} \alpha_{y y'} \cdot \eta_{y'}( x' )$.
When $L = 3$,
\begin{align*}
    \beta_{0}( x' ) &= -c_{10} \cdot \eta_{1}( x' ) - c_{20} \cdot
    \eta_{2}( x' ) \\
    &= -c_{10} \cdot \eta_{1}( x' ) - c_{20} + c_{20} \cdot
    \eta_{0}( x' ) + c_{20} \cdot \eta_{1}( x' ) \\
    &= c_{20} \cdot \eta_{0}( x' ) + (c_{20} - c_{10}) \cdot \eta_{1}( x' ) - c_{20} \\
    \beta_{1}( x' ) &= c_{10} \cdot \eta_{0}( x' ) - c_{21} \cdot \eta_{2}( x' ) \\
    &= c_{10} \cdot \eta_{0}( x' ) - c_{21} + c_{21} \cdot \eta_{0}( x' ) + c_{21} \cdot \eta_{1}( x' ) \\
    &= ( c_{10} + c_{21} ) \cdot \eta_{0}( x' ) + c_{21} \cdot \eta_{1}( x' ) - c_{21} \\
    \beta_{2}( x' ) &= c_{20} \cdot \eta_{0}( x' ) + c_{21} \cdot \eta_{1}( x' ).
\end{align*}
We may rewrite the first two expressions in terms of $\beta_2$:
\begin{align*}
    \beta_0( x' ) &= \beta_2( x' ) + ( c_{20} - c_{10} - c_{21} ) \cdot \eta_1( x' ) - c_{20} \\
    \beta_1( x' ) &= \beta_2( x' ) + ( c_{10} - c_{21} - c_{20} ) \cdot \eta_0( x' ) - c_{21}.
\end{align*}

The claim is that
\begin{align*}
f^*( x ) &= 
\frac{c_{10} \cdot \eta_{1}( x ) + c_{20} \cdot \eta_{2}( x )}{c_{20} \cdot \eta_{0}( x ) + c_{21} \cdot \eta_{1}( x )}  \\
&= \frac{- c_{20} \cdot \eta_{0}( x ) + (c_{10} - c_{20}) \cdot \eta_{1}( x ) + c_{20}}{c_{20} \cdot \eta_{0}( x ) + c_{21} \cdot \eta_{1}( x )}  \\
&= -\frac{\beta_0( x )}{\beta_2( x )}.
\end{align*}

Observe that
\begin{align*}
    \frac{\beta_0( x' )}{\beta_2( x' )} - \frac{\beta_0( x )}{\beta_2( x )} &= 
    (c_{20} - c_{10} - c_{21}) \cdot \left[ \frac{\eta_{1}( x' )}{\beta_{2}( x' )} - \frac{\eta_{1}( x )}{\beta_{2}( x )} \right] - c_{20} \cdot \left[ \frac{1}{\beta_2( x' )} - \frac{1}{\beta_2( x )} \right] \\
    &\propto
    (c_{20} - c_{10} - c_{21}) \cdot \left[ \beta_2( x ) \cdot \eta_1( x' ) - \beta_2( x' ) \cdot \eta_1( x ) \right] - c_{20} \cdot \left[ \beta_2( x ) - \beta_2( x' ) \right],
\end{align*}
where the proportionality factor is $\frac{1}{\beta_2( x ) \cdot \beta_2( x' )} > 0$.
Observe that the first term simplifies to:
\begin{align*}
    & \beta_2( x ) \cdot \eta_1( x' ) - \beta_2( x' ) \cdot \eta_1( x )\\
    ={}& \left[ c_{20} \cdot \eta_{0}( x ) + c_{21} \cdot \eta_{1}( x ) \right] \cdot \eta_1( x' ) - \left[ c_{20} \cdot \eta_{0}( x' ) + c_{21} \cdot \eta_{1}( x' ) \right] \cdot \eta_1( x ) \\
    ={}& c_{20} \cdot \left[ \eta_{0}( x ) \cdot \eta_1( x' ) - \eta_{0}( x' ) \cdot \eta_1( x ) \right].
\end{align*}
Thus,
\begin{align*}
    \frac{\beta_0( x' )}{\beta_2( x' )} - \frac{\beta_0( x )}{\beta_2( x )} &\propto
    (c_{20} - c_{10} - c_{21}) \cdot \left[ \eta_{0}( x ) \cdot \eta_1( x' ) - \eta_{0}( x' ) \cdot \eta_1( x ) \right] - \left[ \beta_2( x ) - \beta_2( x' ) \right].
\end{align*}

Now observe that
\begin{align*}
    & w( x, x' ) - w( x', x )\\
    ={}& \beta_0( x' ) \cdot \eta_0( x ) + \beta_1( x' ) \cdot \eta_1( x ) + \beta_2( x' ) \cdot \eta_2( x ) \\
    ={}& \left[ \beta_2( x' ) + ( c_{20} - c_{10} - c_{21} ) \cdot \eta_1( x' ) - c_{20} \right] \cdot \eta_0( x ) + \\
    &\phantom{=}\,\, \left[ \beta_2( x' ) + ( c_{10} - c_{21} - c_{20} ) \cdot \eta_0( x' ) - c_{21} \right] \cdot \eta_1( x ) + \beta_2( x' ) \cdot \eta_2( x ) \\
    ={}& \beta_2( x' ) \cdot \eta_0( x ) + ( c_{20} - c_{10} - c_{21} ) \cdot \eta_1( x' ) \cdot \eta_0( x ) - c_{20} \cdot \eta_0( x ) + \\
    &\phantom{=}\,\,\, \beta_2( x' )\cdot \eta_1( x ) + ( c_{10} - c_{21} - c_{20} ) \cdot \eta_0( x' ) \cdot \eta_1( x ) - c_{21} \cdot \eta_1( x ) + \beta_2( x' ) \cdot \eta_2( x ) \\    
    ={}& ( c_{20} - c_{10} - c_{21} ) \cdot \left[ \eta_1( x' ) \cdot \eta_0( x ) - \eta_0( x' ) \cdot \eta_1( x ) \right] + \left[ \beta_2( x' ) - \beta_2( x ) \right].
\end{align*}
Thus, $w( x, x' ) - w( x', x ) > 0 \iff \frac{\beta_0( x' )}{\beta_2( x' )} > \frac{\beta_0( x )}{\beta_2( x )}$.
\end{proof}

\begin{proof}[Proof of Proposition~\ref{lemm:bayes-opt-auc-multi-dist}]
Note that
$$ w( x, x' ) = \frac{1}{K} \sum_{k \in [K]} \frac{1}{\pi^{(k)} \cdot (1 - \pi^{(k)})} \cdot \eta^{(k)}( x ) - \frac{1}{K} \sum_{k \in [K]} \frac{1}{\pi^{(k)} \cdot (1 - \pi^{(k)})} \cdot \eta^{(k)}( x ) \cdot \eta^{(k)}( x'), $$
and so
$$ w( x, x' ) > w( x', x ) \iff \frac{1}{K} \sum_{k \in [K]} \frac{1}{\pi^{(k)} \cdot (1 - \pi^{(k)})} \cdot \eta^{(k)}( x ) > \frac{1}{K} \sum_{k \in [K]} \frac{1}{\pi^{(k)} \cdot (1 - \pi^{(k)})} \cdot \eta^{(k)}( x' ). $$
Thus, it follows that the Bayes-optimal scorer takes the form of the average class probabilities.
\end{proof}

\begin{proof}[Proof of Proposition~\ref{prop:pareto-label-agg}]
Suppose $K = 2$, and let $\bar{\mathsf{Y}} = \sum_{k \in [K]} \mathsf{Y}^{(k)} \in \{ 0, 1, \ldots, K \}$.
Note that 
$\bar{\mathsf{Y}} = 0 \iff \mathsf{Y}^{(1)} = 0 \land \mathsf{Y}^{(2)} = 0$, and similarly
$\bar{\mathsf{Y}} = 2 \iff \mathsf{Y}^{(1)} = 1 \land \mathsf{Y}^{(2)} = 1$.
Further suppose that $\mathsf{Y}^{(1)} \independent \mathsf{Y}^{(2)} \mid \mathsf{X}$.
We now have
\begin{align*}
    \bar{\eta}_{0}( x ) &= (1 - \eta^{(1)}( x )) \cdot (1 - \eta^{(2)}( x )) \\
    &= 1 - \eta^{(1)}( x ) - \eta^{(2)}( x )) + \eta^{(1)}( x ) \cdot \eta^{(2)}( x )) \\
    \bar{\eta}_{1}( x ) &= \eta^{(1)}( x ) \cdot (1 - \eta^{(2)}( x )) + (1 - \eta^{(1)}( x )) \cdot \eta^{(2)}( x ) \\
    &= \eta^{(1)}( x ) + \eta^{(2)}( x ) - 2 \cdot \eta^{(1)}( x ) \cdot \eta^{(2)}( x ) \\
    \bar{\eta}_{2}( x ) &= \eta^{(1)}( x ) \cdot \eta^{(2)}( x ).
\end{align*}
By~\citet[Theorem 3]{Uematsu:2015}, 
the Bayes-optimal scorer of AUC with uniform costs against $\bar{\mathsf{Y}}$ will preserve the ordering of
\begin{align*}
    f^*( x ) &= \frac{\bar{\eta}_1( x ) + \bar{\eta}_2( x )}{\bar{\eta}_0( x ) + \bar{\eta}_1( x )} \\
    &= \frac{\eta^{(1)}( x ) + \eta^{(2)}( x ) - \eta^{(1)}( x ) \cdot \eta^{(2)}( x )}{1 - \eta^{(1)}( x ) \cdot \eta^{(2)}( x )}.
\end{align*}

Consider a dataset composed of 6 examples with independently distributed binary scoring functions, with the following distribution $D^{(1)}$ for the first signal: 
$\eta_1(x_1) = 1$,
$\eta_1(x_2) = 0.2$,
$\eta_1(x_3) = 0.62$,
$\eta_1(x_4) = 0.44$,
$\eta_1(x_5) = 0.56$,
$\eta_1(x_6) = 0.81$,
and the following distribution $D^{(2)}$ for the second signal:
$\eta_2(x_1) = 0.44$
$\eta_2(x_2) = 0.56$,
$\eta_2(x_3) = 0.81$,
$\eta_2(x_4) = 1$,
$\eta_2(x_5) = 0.2$,
$\eta_2(x_6) = 0.62$.

The optimal solution for the label aggregation objective is $f_{\rm opt}^{\rm LA}(x_1)=1.78571$, $f_{\rm opt}^{\rm LA}(x_2)=0.72973$, $f_{\rm \rm opt}(x_3)=1.86380$, $f_{\rm opt}^{\rm LA}(x_4)=1.78571$, $f_{\rm opt}^{\rm LA}(x_5)=0.72973$, $f_{\rm opt}^{\rm LA}(x_6)=1.86380$, and yields ${\rm AUC}(f_{\rm opt}^{\rm LA}; D^{(1)})=0.65559$ and ${\rm AUC}(f_{\rm opt}^{\rm LA}; D^{(2)})=0.65559$. 
That solution is however Pareto dominated by any scoring function following the ordering $g$ of the examples: (4, 0, 2, 5, 1, 3), which yields ${\rm AUC}(g; D^{(1)})=0.65706$ and ${\rm AUC}(g; D^{(2)})=0.65862$.
\end{proof}

\begin{proof}[Proof of Proposition~\ref{lemm:bayes-opt-auc-label-agg}]
Let $\bar{\mathsf{Y}} = \sum_k \mathsf{Y}^{(k)} \in \{ 0, 1, \ldots, K\}$ have class probability function $\bar{\eta} \colon \XCal \to \Delta( \{ 0, 1, \ldots, K \} )$. By~\citet[Corollary 1]{Uematsu:2015}, with costs $c_{\bar{y} \bar{y}'} = 1( \bar{y} > \bar{y}'  ) \cdot | \bar{y} - \bar{y}' |$, %
the Bayes-optimal scorer will preserve the ordering of
\begin{align*}
    f^*(x) &= \mathbb{E}\left[\bar{\mathsf{Y}} \mid \mathsf{X}=x\right]
    ~= \sum_{n=0}^K n \cdot \bar{\eta}_n(x) ~= \sum_{n=1}^K n \cdot \bar{\eta}_n(x).
\end{align*}

We will show that the above evaluates to $\sum_{k=1}^K\eta^{(k)}(x)$. We start with the RHS:
\begin{align*}
   & \sum_{k=1}^K\eta^{(k)}(x)\\
    ={}& \sum_{k=1}^K \sum_{\by \in \{0,1\}^{K}} \P\left(
        \sY{1}=y_1, \ldots, \sY{K}=y_{K} \,\Big|\, \sX = x
    \right) \cdot \mathbf{1}(y_k=1)\\
    ={}&
    \sum_{k=1}^K \sum_{\by \in \{0,1\}^{K}} 
    \P\left(
        \sY{1}=y_1, \ldots, \sY{K}=y_{K}  \,\Big|\, \sX = x
    \right) \cdot \mathbf{1}(y_k=1) \cdot\1\Big(\textstyle\sum_j y_j \geq  1\Big)\\
    ={}&\sum_{k=1}^K \sum_{\by \in \{0,1\}^{K}} 
    \P\left(
        \sY{1}=y_1, \ldots, \sY{K}=y_{K}  \,\Big|\, \sX = x
    \right) \cdot \mathbf{1}(y_k=1) \cdot \sum_{n=1}^{K}\1\Big(\textstyle\sum_j y_j = n\Big)\\
    ={}& \sum_{k=1}^K \sum_{n=1}^{K}  \sum_{\by \in \{0,1\}^{K}} 
    \P\left(
        \sY{1}=y_1, \ldots, \sY{K}=y_{K}  \,\Big|\, \sX = x
    \right) \cdot \mathbf{1}(y_k=1)\cdot \1\Big(\textstyle\sum_{j} y_j = n\Big)\\
     ={}& \sum_{k=1}^K \sum_{n=1}^{K} \sum_{\by \in \{0,1\}^{K}} 
    \P\left(
        \sY{1}=y_1, \ldots, \sY{K}=y_{K}  \,\Big|\, \sX = x
    \right) \cdot \1\Big(\textstyle\sum_{j} y_j = n\Big)\cdot \mathbf{1}(y_k=1)\\
     ={}&  \sum_{n=1}^{K} \sum_{\by \in \{0,1\}^{K}}  \P\left(
        \sY{1}=y_1, \ldots, \sY{K}=y_{K}  \,\Big|\,\sX = x
    \right) \cdot \1\Big(\textstyle\sum_{j} y_j = n\Big)\cdot \sum_{k=1}^K\mathbf{1}(y_k=1)\\
    ={}&  \sum_{n=1}^{K} \sum_{\by \in \{0,1\}^{K}} \P\left(
        \sY{1}=y_1, \ldots, \sY{K}=y_{K}  \,\Big|\, \sX = x
    \right) \cdot \1\Big(\textstyle\sum_{j} y_j = n\Big)\cdot n\\
    ={}& \sum_{n=1}^{K} n \cdot\sum_{\by \in \{0,1\}^{K}}  \P\left(
        \sY{1}=y_1, \ldots, \sY{K}=y_{K}  \,\Big|\, \sX = x
    \right) \cdot \1\Big(\textstyle\sum_{j} y_j = n\Big)\\
    ={}& \sum_{n=1}^{K} n \cdot \P\left(
        \bar{\mathsf{Y}} = n  \,\Big|\, \sX = x
    \right)\\
    ={}& \sum_{n=1}^K n \cdot \bar{\eta}_n(x).
\end{align*}

Since the optimal scorer has the same form as that for the loss aggregated AUC in Proposition \ref{lemm:bayes-opt-auc-multi-dist} when $a_k = \pi^{(k)} \cdot (1 - \pi^{(k)}), \forall k$, and we know that the optimal scorers for the loss aggregated AUC are Pareto optimal from Proposition \ref{prop:loss-aggregated-pareto}, it follows that the above solution is also Pareto optimal. 
\end{proof}

\begin{proof}[Proof of Proposition \ref{prop:dictatorship}]
We know from Proposition \ref{lemm:bayes-opt-auc-multi-dist} that the optimal scorer will preserve the ordering of $f^*(x) = \alpha^{(1)} \cdot\eta^{(1)}(x) + \alpha^{(2)} \cdot\eta^{(2)}(x)$. 

We start with the case where $\alpha^{(1)} > \alpha^{(2)}$. Since the labels are deterministic $\eta^{(1)}(x), \eta^{(2)}(x) \in \{0,1\}$. Hence when $\eta^{(1)}( x ) > \eta^{(1)}( x' )$, we have that $\eta^{(1)}( x ) = 1$ and $\eta^{(1)}( x' ) = 0$. Therefore 
\begin{align*}
f^*(x) ={}& \alpha^{(1)} + \alpha^{(2)} \cdot \eta^{(2)}( x ) \geq \alpha^{(1)} > \alpha^{(2)} \geq \alpha^{(2)} \cdot \eta^{(2)}( x' ) \\
={}& \alpha^{(1)} \cdot\eta^{(1)}(x') + \alpha^{(2)} \cdot\eta^{(2)}(x') = f^*(x').
\end{align*}
Similarly, when $\alpha^{(2)} > \alpha^{(1)}$ and $\eta^{(2)}( x ) > \eta^{(2)}( x' )$, we have $\eta^{(2)}( x ) = 1$ and $\eta^{(2)}( x' ) = 0$, and therefore:
\begin{align*}
f^*(x) ={}& \alpha^{(1)} \cdot \eta^{(1)}( x ) + \alpha^{(2)}  \geq \alpha^{(2)} > \alpha^{(1)} \geq \alpha^{(1)} \cdot \eta^{(1)}( x' ) \\
={}& \alpha^{(1)} \cdot\eta^{(1)}(x') + \alpha^{(2)} \cdot\eta^{(2)}(x') = f^*(x'),
\end{align*}
which completes the proof.
\end{proof}

\begin{proof}[Proof of Proposition \ref{prop:label-agg-deterministic}]
The proof for ordinal costs $c_{\bar{y} \bar{y}'} = 1( \bar{y} > \bar{y}'  ) \cdot | \bar{y} - \bar{y}' |$ follows directly from the more general case in Proposition \ref{lemm:bayes-opt-auc-label-agg}.

We now provide the proof for costs $c_{\bar{y} \bar{y}'} = 1$. By Lemma \ref{lemm:auc-multipartite-opt}, since $K=3$ ($\bar{\mathsf{Y}}$ can be $0, 1$ or $2$), the Bayes-optimal scorer will preserve the ordering of
\begin{align*}
   & f^*( x ) \\
    ={}& 
    \frac{\bar{\eta}_1( x ) +  \bar{\eta}_2( x )}{\bar{\eta}_0( x ) + \bar{\eta}_1( x )}\\
    ={}&
    \frac{
     \P\big(\sY{1} = 0, \sY{2} = 1 \,\big|\, \sX = x) + \P\big(\sY{1} = 1, \sY{2} = 0 \,\big|\, \sX = x) + \P\big(\sY{1} = 1, \sY{2} = 1 \,\big|\, \sX = x) }{ 
     \P\big(\sY{1} = 0, \sY{2} = 0 \,\big|\, \sX = x) + \P\big(\sY{1} = 0, \sY{2} = 1 \,\big|\, \sX = x) + \P\big(\sY{1} = 1, \sY{2} = 0 \,\big|\, \sX = x)
    }.
\end{align*}
We will show that $f^*(x)$ is a strictly monotonic transformation of: 
\begin{align*}
\eta^{(1)}(x) + \eta^{(2)}(x) ={}& \P\big(\sY{1} = 0, \sY{2} = 1 \,\big|\, \sX = x) + \P\big(\sY{1} = 1, \sY{2} = 0 \,\big|\, \sX = x)\\
&+ 2 \cdot \P\big(\sY{1} = 1, \sY{2} = 1 \,\big|\, \sX = x).
\end{align*}

Since the labels are deterministic, $\P\big(\sY{1} = i, \sY{2} = j \,\big|\, \sX = x) = 1$ for exactly one particular $(i,j)$. We consider all four possible cases:

(1) $\P\big(\sY{1} = 0, \sY{2} = 0 \,\big|\, \sX = x) = 1$. We have: $f^*( x ) = 0$ and $\eta^{(1)}(x) + \eta^{(2)}(x) = 0$.

(2) $\P\big(\sY{1} = 0, \sY{2} = 1 \,\big|\, \sX = x) = 1$. We have: $f^*( x ) = 1$ and $\eta^{(1)}(x) + \eta^{(2)}(x) = 1$.

(3) $\P\big(\sY{1} = 1, \sY{2} = 0 \,\big|\, \sX = x) = 1$. We have: $f^*( x ) = 1$ and $\eta^{(1)}(x) + \eta^{(2)}(x) = 1$.

(4) $\P\big(\sY{1} = 1, \sY{2} = 1 \,\big|\, \sX = x) = 1$. We have: $f^*( x ) = \infty$ and $\eta^{(1)}(x) + \eta^{(2)}(x) = 2$.

Clearly, $f^*(x)$ is a strictly monotonic transformation of $\eta^{(1)}(x) + \eta^{(2)}(x).$
\end{proof}

\begin{proof}[Proof of Lemma~\ref{lem:prod-aggregation}]
The first statement (a) follows directly from Proposition \ref{lemm:bayes-opt-auc-multi-dist}.

For the second statement (b), let $\bar{\mathsf{Y}} = \prod_k \mathsf{Y}^{(k)} \in \{ 0, 1\}$ have class probability function $\bar{\eta}_{\rm prod} \colon \XCal \to [0,1]$. Applying Lemma \ref{lemm:auc-bipartite-opt}, the Bayes-optimal scorer $f^*_{\rm prod}$ill preserve the ordering of:
\[
\gamma_{\rm prod}(x) = \bar{\eta}_{\rm prod}(x) = \P(\sY{1}=1,\ldots, \sY{K}=1 \,|\, \mathsf{X}=x). 
\]

For the third statement (c), given that the labels are deterministic, we note that $\gamma_{\rm prod}(x) > \gamma_{\rm prod}(x')$ implies that $\P(\sY{1}=1,\ldots, \sY{K}=1 \,|\, \mathsf{X}=x) = 1$, while $\P(\sY{1}=1,\ldots, \sY{K}=1 \,|\, \mathsf{X}=x') = 0$. This indicates that $\gamma_{\rm sum}(x) = \sum_k \eta^{(k)}(x) = K$ and $\gamma_{\rm sum}(x') = \sum_k \eta^{(k)}(x') < K$. As a result, $\gamma_{\rm sum}(x) > \gamma_{\rm sum}(x')$.
\end{proof}

\section{Additional theoretical results}
\label{app:additional-results}
\begin{theorem}\label{thm:auc-gap}
Suppose that the labels are binary, i.e., $\mathsf{Y}^{(k)}\in\{0,1\}$ for all $k\in\mathbb{N}_{\ge1}$, and $a_{k}>0$. Assume further that the labels $\{\mathsf{Y}^{(k)}\}_{k\in\mathbb{N}_{\ge1}}$ are jointly independent conditioned on $\mathsf{X}$. Consider the aggregation function $\psi(\mathsf{Y}^{(1)}, \ldots, \mathsf{Y}^{(K)}) = \sum_{k\in [K]}  a_k \mathsf{Y}^{(k)}$. Define
\begin{align*}
    f^{*} &\defEq \underset{f:\XCal\to\mathbb{R}}{\argmax} \, \operatorname{AUC}_\textnormal{LA}\left(f;\{D^{(k)}\}\right)\,,\\
    \tilde{f}(x) &\defEq \phi\left(\sum_{k\in[K]}a_{k}\eta_{k}(x)\right)\,,
\end{align*}
where $\phi:\mathbb{R}\to\mathbb{R}$ is a strictly increasing function. Then,

\resizebox{0.85\columnwidth}{!}{
\begin{minipage}{\linewidth}
\begin{equation*}
\begin{aligned}
& \operatorname{AUC}_\textnormal{LA}\left(f^{*};\{D^{(k)}\}\right)-\operatorname{AUC}_\textnormal{LA}\left(\tilde{f};\{D^{(k)}\}\right) \\
\le{}& \psi\left(\mathbb{E}_{\mathsf{X}_1,\mathsf{X}_2}\left[\frac{\sum_{k\in[K]} a_{k}^{3}\sum_{i\in\{1,2\}}\eta_{k}(\mathsf{X}_{i})\left(1-\eta_{k}(\mathsf{X}_{i})\right)}{\left(\sum_{k\in[K]} a_{k}^{2}\sum_{i\in\{1,2\}}\eta_{k}(\mathsf{X}_{i})\left(1-\eta_{k}(\mathsf{X}_{i})\right)\right)^{3/2}}\right]\right),
\end{aligned}
\end{equation*}
\end{minipage}
}

where $\psi(t)\triangleq\frac{2t}{1-t}$, and $\mathsf{X}_1$ and $\mathsf{X}_2$ are i.i.d. copies of $\mathsf{X}$.
\end{theorem}

\begin{corollary}\label{cor:auc-gap}
If $a_{k}=1$ for all $k$ and $\eta_{k}(x)$ is bounded away from 0 and 1, i.e., there exists $c\in(0,1/2)$ such that $\eta_{k}(x)\in[c,1-c]$ for all $k\in\mathbb{N}_{\ge 1}$ and for all $x\in\mathcal{X}$, then the expected AUC difference between the optimal prediction function $f^{*}$ and $\tilde{f}$ converges to zero with a rate of $O(1/\sqrt{K})$:
\begin{equation*}
 \operatorname{AUC}_\textnormal{LA}\left(f^{*};\{D^{(k)}\}\right)-\operatorname{AUC}_\textnormal{LA}\left(\tilde{f};\{D^{(k)}\}\right)  = O\left(\frac{1}{\sqrt{K}}\right)\,.
\end{equation*}
\end{corollary}

\input{proof_AUC_LA_asymp}

\section{Further Analysis of Label Aggregation}
\label{app:further_label_agg_analysis}

This section addresses further properties of label aggregation, including its behavior with anti-correlated labels and its sensitivity to mixing weights, expanding on the discussions in \cref{s:label_agg}.

\subsection{Behavior with Anti-correlated Labels}
Our theory (\cref{lemm:bayes-opt-auc-label-agg}) shows the optimal scorer for label aggregation (cost $c_{\bar{y}\bar{y}^{\prime}} = |\bar{y} - \bar{y}^{\prime}|$) ranks by $f^*(x) \propto \sum_k \eta^{(k)}(x)$. For $K=2$, this is equivalent to ranking by $\delta(x) = p(\sY{1}=1, \sY{2}=1 \mid \sX = x) - p(\sY{1}=0, \sY{2}=0 \mid \sX = x)$. The derivation below shows $f^*(x) \propto \delta(x)$.

From \cref{lemm:bayes-opt-auc-label-agg}, the scorer is $\eta^{(1)}(x)+\eta^{(2)}(x)$. We have $\eta^{(1)}(x) = p(\sY{1}=1, \sY{2}=0 \mid \sX = x) + p(\sY{1}=1, \sY{2}=1 \mid \sX = x)$ and $\eta^{(2)}(x) = p(\sY{1}=0, \sY{2}=1 \mid \sX = x) + p(\sY{1}=1, \sY{2}=1 \mid \sX = x)$. Summing them gives $\eta^{(1)}(x) + \eta^{(2)}(x) = p(\sY{1}=1, \sY{2}=0 \mid \sX = x) + p(\sY{1}=0, \sY{2}=1 \mid \sX = x) + 2 p(\sY{1}=1, \sY{2}=1 \mid \sX = x)$. Since $p(\sY{1}=1, \sY{2}=0 \mid \sX = x) + p(\sY{1}=0, \sY{2}=1 \mid \sX = x) + p(\sY{1}=1, \sY{2}=1 \mid \sX = x) + p(\sY{1}=0, \sY{2}=0 \mid \sX = x) = 1$, the sum simplifies to $1 - p(\sY{1}=0, \sY{2}=0 \mid \sX = x) + p(\sY{1}=1, \sY{2}=1 \mid \sX = x) = 1 + \delta(x)$.

This shows the optimal scorer \emph{does} depend on label correlations via $\delta(x)$. Let's consider the implications:
\begin{itemize}
    \item \textbf{Overlapping ($\sY{1}=\sY{2}$):} $\delta(x) \propto \eta_1(x)$, so the ranking uses $\eta_1(x)$, which is sensible.
    \item \textbf{Anti-correlated ($\sY{1} = 1-\sY{2}$):} $\delta(x) = 0$. The scorer $f^*(x)$ is constant, yielding no ranking. This indicates reduced robustness when the aggregated label $\sY{1}+\sY{2}=1$ is constant, making the AUC objective ill-defined.
    \item \textbf{Mild anti-correlation:} Here, both $p(\sY{1}=1, \sY{2}=1 \mid \sX = x)$ and $p(\sY{1}=0, \sY{2}=0 \mid \sX = x)$ would be small, and the ranking depends on their difference via $\delta(x)$. E.g., when $\delta(x) > 0$, it is more likely that both labels are 1 than 0, resulting in $x$ being ranked higher. The behavior would depend on the exact conditional probabilities.
\end{itemize}

\subsection{Sensitivity of Label Aggregation to Mixing Weights}
We next generalize label aggregation by using non-uniform weights $\alpha_1, \alpha_2 > 0$ for the aggregation function, i.e., $\bar{Y} = \alpha_1 \cdot \sY{1} + \alpha_2 \cdot \sY{2}$, again for the $K=2$ case for simplicity.

Following assumptions of \cref{lemm:bayes-opt-auc-label-agg} (that the Bayes-optimal scorer for costs $c_{\bar{y}\bar{y}^{\prime}} = |\bar{y} - \bar{y}^{\prime}|$ preserves the ordering of $E[\bar{Y}|X=x]$), the Bayes-optimal scorer $\bar{\eta}(x) = E[\bar{Y}|X=x]$ becomes:
\begin{align*}
\bar{\eta}(x) &= 0 \cdot p(\sY{1}=0, \sY{2}=0 \mid \sX = x) + \alpha_1 \cdot p(\sY{1}=1, \sY{2}=0 \mid \sX = x) \\
            &\quad + \alpha_2 \cdot p(\sY{1}=0, \sY{2}=1 \mid \sX = x) + (\alpha_1+\alpha_2) \cdot p(\sY{1}=1, \sY{2}=1 \mid \sX = x) \\
            &= \alpha_1 \cdot [ p(\sY{1}=1, \sY{2}=0 \mid \sX = x) + p(\sY{1}=1, \sY{2}=1 \mid \sX = x) ] \\
            &\quad + \alpha_2 \cdot [ p(\sY{1}=0, \sY{2}=1 \mid \sX = x) + p(\sY{1}=1, \sY{2}=1 \mid \sX = x) ] \\
            &= \alpha_1 \eta^{(1)}(x) + \alpha_2 \eta^{(2)}(x)
\end{align*}
This result shows that the optimal scorer for weighted label aggregation is a direct linear combination of the individual class-probability functions $\eta^{(k)}(x)$, weighted by the \textit{explicitly chosen} aggregation weights $\alpha_k$.

This contrasts sharply with the optimal scorer for loss aggregation (\cref{lemm:bayes-opt-auc-label-agg}), which is $\sum_{k}\frac{a_{k}}{\pi^{(k)}\cdot(1-\pi^{(k)})}\cdot\eta^{(k)}(x)$. In loss aggregation, the effective weight on $\eta^{(k)}(x)$ depends not only on the chosen mixing weight $a_k$ but also implicitly on the label prior $\pi^{(k)}$.

Therefore, regarding sensitivity:
\begin{itemize}
    \item The label aggregation optimal scorer (under the assumptions above) is sensitive to the choice of weights $\alpha_k$ in a direct and predictable way: the final scorer is exactly the $\alpha_k$-weighted sum of the $\eta_k(x)$. It is notably insensitive to the class priors $\pi^{(k)}$.
    \item The loss aggregation optimal scorer is sensitive to both the chosen weights $a_k$ and the class priors $\pi^{(k)}$ through the $\pi^{(k)}(1-\pi^{(k)})$ term.
\end{itemize}

\section{Further Discussion on Application Scenarios}\label{app:application_scenarios}

The problem of synthesizing multiple binary labels into a single coherent ranking has widespread applications beyond the initial examples of information retrieval and medical diagnosis mentioned in \S\ref{s:introduction}. Our analysis of loss versus label aggregation, particularly regarding Pareto optimality and the ``label dictatorship'' issue, can inform choices in various real-world systems. Below, we elaborate on some potential application scenarios:

\begin{itemize}
    \item \textbf{Multi-faceted Information Retrieval:} Users' information needs can be complex, requiring documents to be relevant across multiple dimensions or interpretations. For example, a search system might need to rank documents based on topical relevance to different aspects of a query, alongside signals for document freshness, geographical relevance for location-sensitive queries, or authoritativeness of the source. Each of these aspects can be represented by a (potentially binary) label (e.g., ``\texttt{is\_fresh},'' ``\texttt{is\_geo\_relevant}''). Effectively combining these signals is crucial for user satisfaction~\citep{perkio2005multi}.

    \item \textbf{Recommendation Systems:} Modern recommender systems often optimize for multiple objectives simultaneously. For instance, beyond predicting user engagement (e.g., click-through rate, purchase probability), systems may also aim to promote relevance to users' long-term interests, ensure diversity in recommendations to avoid filter bubbles, or uphold fairness considerations across different item providers or content creators~\citep{zheng2022survey}.

    \item \textbf{Computational Advertising:} Ranking advertisements effectively requires balancing predictions from multiple models. Commonly, platforms consider both the predicted click-through rate (CTR) and the predicted conversion rate (CVR) --- the likelihood that a click leads to a desired action (e.g., a sale). The final ranking aims to optimize overall platform revenue or advertiser value, which often involves a synthesis of these distinct predictive labels~\citep{wang2023towards}. 
\end{itemize}
In each of these scenarios, understanding how different aggregation strategies (loss versus label) behave, as explored in our work, can lead to more principled system design. 
The choice of aggregation method can influence how trade-offs between objectives are handled and whether certain objectives inadvertently dominate others.

\section{Illustrative Example of Undesirable Dictatorship in Loss Aggregation}
\label{app:ir_dictatorship_example}

\S\ref{s:linear_scalarization} discusses how the Bayes-optimal scorer for loss aggregation (\cref{lemm:bayes-opt-auc-multi-dist}) weights individual class-probability functions $\eta^{(k)}(x)$ by $a_k / [\pi^{(k)}(1-\pi^{(k)})]$. This means that even with uniform explicit weights $a_k=1$, labels that are more marginally skewed (i.e., $\pi^{(k)}$ is far from 0.5) receive higher effective weighting. We clarify that $\pi^{(k)} = \P(\sY{k}=1)$ is the marginal prevalence of a label, which is distinct from its conditional probability $\eta^{(k)}(x) = \P(\sY{k}=1\mid \sX=x)$ for a specific instance $x$. A label can have a balanced marginal prior ($\pi^{(k)} \approx 0.5$) yet be perfectly deterministic (noise-free) conditionally (e.g., $\eta^{(k)}(x) \in \{0,1\}$).

\cref{prop:dictatorship} formalizes the problem whereby one label can become a ``dictator'', i.e., the ranking overly favors it. 
Below is an illustrative example from information retrieval where such dictatorship, driven by marginal skewness, would be undesirable.

Consider ranking documents based on two binary, conditionally deterministic (noise-free, $\eta^{(k)}(x) \in \{0,1\}$) labels:
\begin{itemize}
    \item $\mathsf{Y}^{(1)}$: ``\texttt{is\_relevant}'' (core topical relevance to the user's query).
    \item $\mathsf{Y}^{(2)}$: ``\texttt{is\_recent}'' (document published in the last 24 hours).
\end{itemize}
Assume the following marginal priors in the document collection for a typical query:
\begin{itemize}
    \item For ``\texttt{is\_relevant}'' ($\mathsf{Y}^{(1)}$): A good number of documents are relevant, and many are not. Let $\pi^{(1)} = 0.4$. Thus, $\pi^{(1)}(1-\pi^{(1)}) = 0.4 \times 0.6 = 0.24$.
    \item For ``\texttt{is\_recent}'' ($\mathsf{Y}^{(2)}$): Very few documents are extremely recent. Let $\pi^{(2)} = 0.01$. Thus, $\pi^{(2)}(1-\pi^{(2)}) = 0.01 \times 0.99 = 0.0099$.
\end{itemize}
If we use loss aggregation with uniform explicit weights $a_1=1$ and $a_2=1$, the $\alpha^{(k)}$ terms from \cref{prop:dictatorship} become:
\begin{itemize}
    \item $\alpha^{(1)} = a_1 / [\pi^{(1)}(1-\pi^{(1)})] = 1 / 0.24 \approx 4.17$.
    \item $\alpha^{(2)} = a_2 / [\pi^{(2)}(1-\pi^{(2)})] = 1 / 0.0099 \approx 101.01$.
\end{itemize}
Since $\alpha^{(2)} \gg \alpha^{(1)}$, according to \cref{prop:dictatorship}, the ``\texttt{is\_recent}'' label ($\mathsf{Y}^{(2)}$) will dictate the ranking. The system will primarily rank documents based on whether they are extremely recent. An irrelevant but very recent document would likely be ranked above a highly relevant document that is not from the last 24 hours.

This outcome is undesirable if the user's primary goal is to find relevant information, with recency being a secondary, tie-breaking, or less critical factor. The loss aggregation objective, in this case, inadvertently prioritizes the ``\texttt{is\_recent}'' signal heavily, not because its conditional signal is necessarily stronger or more important for the user's core task, but because its marginal distribution is highly skewed. The system optimizes for a property of the label's distribution in the dataset (its rarity) rather than strictly adhering to an equal consideration of the (equally clean) conditional signals for relevance and recency. This highlights why relying on marginal priors for weighting can be problematic. Label aggregation (\cref{lemm:bayes-opt-auc-label-agg}, with appropriate costs) would score by $\eta^{(1)}(x) + \eta^{(2)}(x)$, giving equal intrinsic weight to the conditional signals from $\mathsf{Y}^{(1)}$ and $\mathsf{Y}^{(2)}$ regardless of their marginal priors.

\section{Empirical details}

In this section we provide details to empirical experiments in the paper. 

In Table~\ref{table:prompts_msmarco} we summarize the prompts used to obtain relevance and clicks labels for the MS MARCO examples shown in Tables~\ref{table:example_msmarco1}.

\begin{table*}
  \centering
    \begin{tabular}{p{0.15\linewidth} p{0.80\linewidth}}
    \toprule
    {\textbf{Label}} & {\textbf{Input prompt}}\\
    \midrule
    Engagement &Imagine you are an information retrieval expert.
        Your aim is to judge whether the following document will be clicked if it is shown for the following query.

    Query: \{query\}\\
    & Document: \{document\}\\

    & Do you predict this document will be clicked for this query?
    Output only yes or no and nothing else.\\ 
    \midrule
    Relevance & Imagine you are an information retrieval expert.
        Your aim is to judge whether the following document will be considered relevant by a normal user if it is shown for the following query.

    Query: \{query\}\\
    & Document: \{document\}\\

    & Do you predict this document will be considered relevant by a normal user for this query?
    Output only yes or no and nothing else.\\
    \bottomrule
    \end{tabular}
    \caption{Prompts used for predicting engagement and relevance of different documents per query in MS MARCO dataset~\cite{bajaj2018msmarcohumangenerated}. We select documents according to the original order from the MS MARCO dataset.}
    \label{table:prompts_msmarco}
\end{table*}

\subsection{Synthetic data generation for the exhaustive enumeration of the optimal solutions}
\label{s:synth_details_brute_force}
We begin with 
a synthetic experiment with a goal of verifying our theory pertaining to maximizers of the loss aggregation objectives.
We generate a synthetic dataset, 
and 
compute the optimal solutions via brute force evaluation;

We consider a two dimensional Gaussian distribution with zero mean and a uniformly sampled covariance matrix $\Sigma$.
We sample $N$ two-dimensional samples and consider the first dimension to correspond to $y_1(x)$ (clicks) and the second dimension to $y_2(x)$ (relevance).
We make the two scores bi-level by thresholding the numbers at $0$.

We aim to consider an exhaustive set of hypotheses such that it is possible to conduct a brute force evaluation of different objectives. To this end, we consider the following hypothesis class $f$.
We consider the image of $f$ to be $[P] \defEq \{ 1, 2, \ldots, P \}$. 
For each sample $x$ and the corresponding relevance scores $y_1(x)$ and $y_2(x)$, if both relevance scores agree, we set $f(x)$ to correspond to $P \cdot y_1(x)$ (i.e., the maximum value if both scores equal $1$, and $0$ if both scores equal $0$).
For all remaining $x$ (where $y_1(x) \neq y_2(x)$), we consider all assignments of $f$ from $[P]$.
This way, if $y_1(x) \neq y_2(x)$ on $M$ samples, the hypothesis class $\hat{\mathscr{Y}}$ for $f$ consists of $P^M$ hypotheses.

We enumerate different values for the $\alpha_1$ and $\alpha_2$ weights in the loss aggregation objective from a cartesian product $[5] \times [5]$, and for each pair of values, we exhaustively enumerate all $f \in \hat{\mathscr{Y}}$ and calculate: the loss aggregation objective $\sum_{k\in [K]} a_k \auc(f; y_k)$ and both of the per scoring function objectives $\auc(f; y_1)$ and $\auc(f; y_2)$.
We then do the same for label aggregation and label product objectives.

In the end, we plot the maximizers of loss and label aggregation in Figure~\ref{tbl:pareto_fronts}, and confirm the following characterizations of the optimal solutions:
\begin{enumerate}[label=(\alph*),itemsep=0pt,topsep=0pt,leftmargin=16pt]
    \item $\mathscr{Y}^*_{\rm LP}$ consists of all assignments of scores from $\{0, 1, 2\}$ to the disagreeing examples (they can disagree in any way, and only need to be smaller than $3$).
    \item $\mathscr{Y}^*_{{\rm LoA}, <}$ consists of only a single assignment to $f$ to the disagreeing examples such that: $f$ for examples with $y_1(x)>y_2(x)$ is $1$ and otherwise it is $2$. 
    \item $\mathscr{Y}^*_{\rm LoA, >}$ consists of only a single assignment to $f$ to the disagreeing examples such that: $f$ for examples with $y_1(x)>y_2(x)$ is $2$ and otherwise it is $1$. 
    \item $\mathscr{Y}^*_{\rm LoA, =}$ consists of all assignments to $f$ to the disagreeing examples such that any relation is satisfied between the two sets of disagreeing examples. This case is equivalent to LA.
\end{enumerate}

\begin{figure}[!t]
\centering
        \resizebox{0.7\linewidth}{!}{
    \includegraphics[scale=0.3]{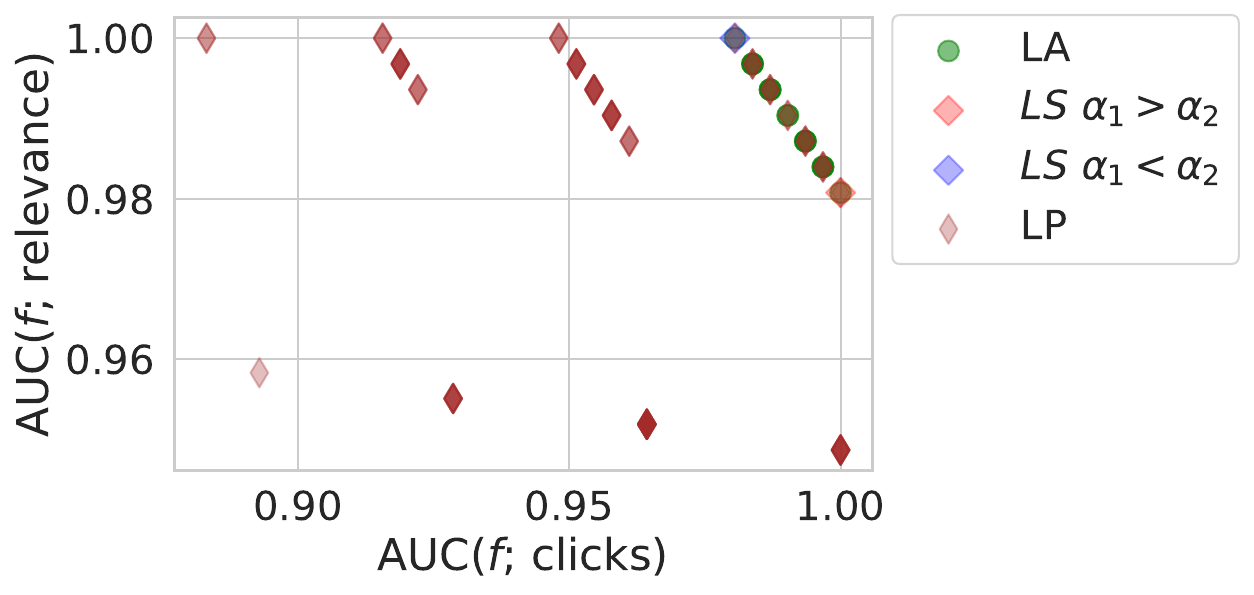}
    }
    \caption{AUC metrics for Bayes-optimal solutions according to different objectives.
    We confirm the following relations among the per-objective optimal solutions sets: 
$
\mathscr{Y}^*_{{\rm LoA}, >} \subset \mathscr{Y}^*_{{\rm LoA}, =} = \mathscr{Y}^*_{\rm LaA}  \subset \mathscr{Y}^*_{\rm LP}
$
and
$
\mathscr{Y}^*_{{\rm LoA}, <} \subset \mathscr{Y}^*_{{\rm LoA}, =} = \mathscr{Y}^*_{\rm LaA}  \subset \mathscr{Y}^*_{\rm LP}
$. Notice the Pareto front is given by a linear function spanned by the solutions corresponding to $
\mathscr{Y}^*_{{\rm LoA}, <}$ and $
\mathscr{Y}^*_{{\rm LoA}, >}$. 
    }
    \label{tbl:pareto_fronts}
\end{figure}

\subsection{Details on AUC optimization}
\label{s:auc_optimization}
The indicator function in the AUC definition~\eqref{eqn:auc-bp-conditional} makes it non-differentiable, and thus, direct AUC optimization is challenging.
To mitigate this, previous works propose relaxation of the indicator function with different surrogate functions, including the hinge loss function, the logistic loss function and other choices~\citep{sun2023enhancing, tang2022smooth}.
Thus, for a surrogate function $\phi$, we arrive at the following formulation:

$$\text{AUC}^{\phi}( f; D ) \,=\, \mathbb{E}_{\mathsf{X}^+ \mid \mathsf{Y} = 1}\mathbb{E}_{\mathsf{X}^- \mid \mathsf{Y} = 0}\Big[\phi( f(\mathsf{X}^+) - f(\mathsf{X}^-) ) \Big]$$

In our work, unless otherwise stated, we assume a logistic surrogate function $\phi$. %

\subsection{Details on synthetic training experiments}
\label{s:synth_training_experiments}
We  draw uniform two-dimensional random vectors $w_1$, $w_2$ over $[-1,1]^2$ and calculate label distributions
using a sigmoid linear model, with $\eta^{(1)}(x) = \sigma(\alpha \cdot w_1^\top x)$ and $\eta^{(2)}(x) = \sigma(\tau \cdot w_2^\top x)$, where $\sigma(z) = \frac{1}{1+\exp(-z)}$ is the sigmoid function.
We then sample labels $y_1(x)$ and $y_2(x)$ uniformly from the distributions $\eta^{(1)}(x)$ and $\eta^{(2)}(x)$.

We train a 3-layer MLP model with hidden dimension 256 and ReLU activation over 8K examples for 50 epochs. 
We evaluate on held out 2K examples.

\subsection{Additional results}
For completeness, in Figure~\ref{fig:pi2_sweep_appendix} we report metrics for per-signal AUC corresponding to results in Figure~\ref{fig:pi2_sweep}.

\begin{figure*}[!t]
    \centering
    
    {\includegraphics[width=0.8\textwidth]{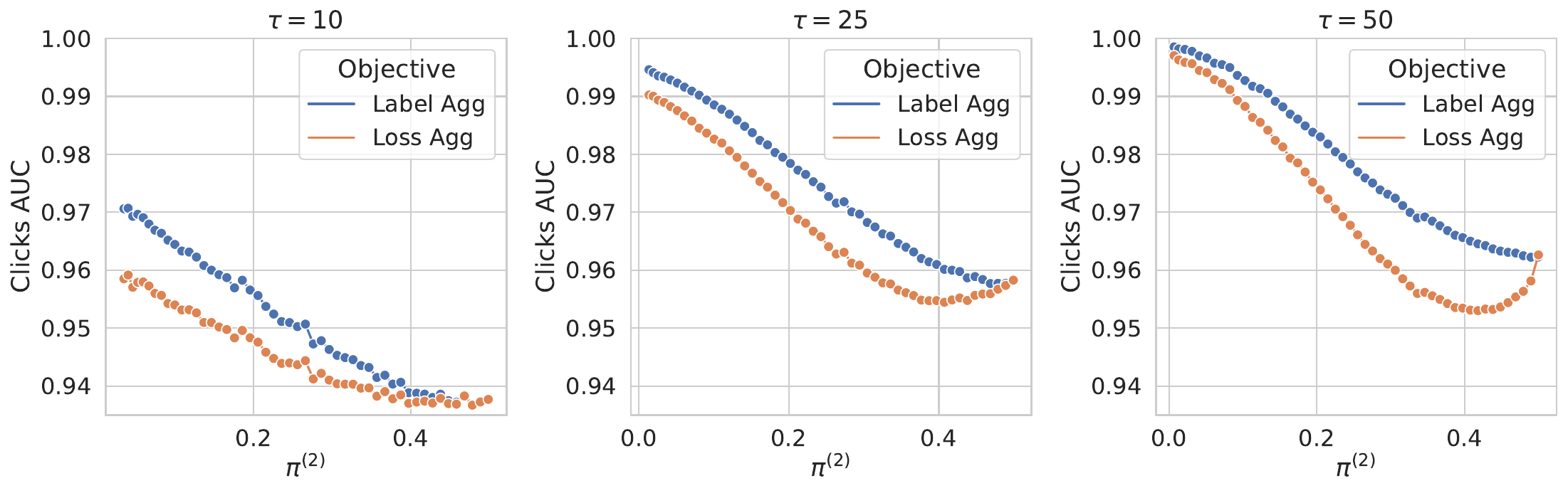}}
    {\includegraphics[width=0.8\textwidth]{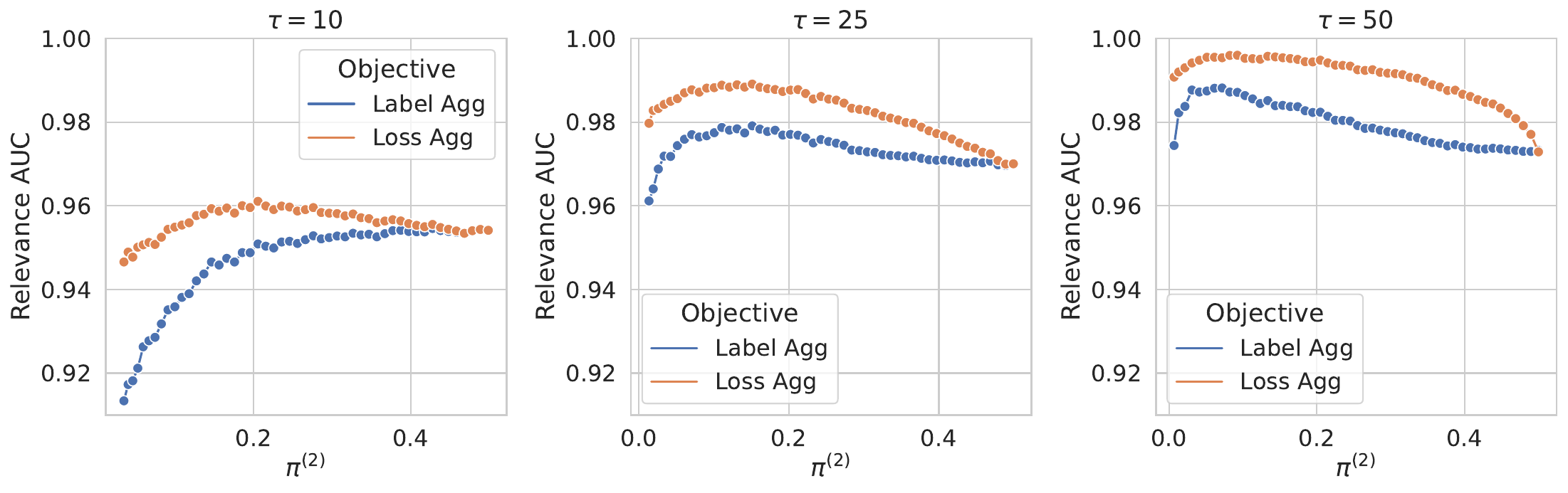}}
    \vspace{-0.3cm}
    \caption{Plots of $\text{AUC}(\cdot; D^{(1)})$ (i.e., clicks) and $\text{AUC}(\cdot; D^{(2)})$ (i.e., relevance) for optimal scorers  as a function of skewness in label $\sY{2}$, corresponding to to results in Figure~\ref{fig:pi2_sweep}. We compare label aggregation and loss aggregation on a synthetic dataset. We fix $\pi^{(1)} = \mathbf{P}(\mathsf{Y}^{(1)}=1) = 0.5$ and vary $\pi^{(2)} = \mathbf{P}(\mathsf{Y}^{(2)}=1)$. The sigmoid scaling parameter $\tau$ controls how close the label distribution is to a deterministic distribution (the larger $\tau$, the closer it is to a deterministic distribution). We find loss aggregation to lead to a higher difference between per-label AUC metrics.\label{fig:pi2_sweep_appendix}}
\end{figure*}

\subsection{Additional synthetic experiments}
\begin{table}[!t]
    \centering
    \ifarxivversion
        \begin{tabular}{@{}lccccc@{}}
            \toprule
            \toprule
            \textbf{Objective} & \textbf{$\pi^{(2)}=0.5$} & \textbf{$\pi^{(2)}=0.6$} & \textbf{$\pi^{(2)}=0.7$} & \textbf{$\pi^{(2)}=0.8$} & \textbf{$\pi^{(2)}=0.9$}\\ %
            \toprule
                AUC($y_1$) & 0.53	&0.44&	0.52&	0.56	&0.50\\
                \midrule
                
                 AUC($y_2$) &0.49	&0.46&	0.50&	0.54	&0.45\\
                \midrule
                
                AUC$_{\rm LaA}$  &0.06&	\best{0.04}&	\best{0.08}	&\best{0.07}&	\best{0.07} \\
                
                \midrule
                AUC$_{\rm LoA}^{(1, 1)}$  & \best{0.05}	&0.13&	0.09&	0.11&	0.10\\
                \bottomrule
        \end{tabular}%
    \else
        \renewcommand{\arraystretch}{1}
        \resizebox{0.5\linewidth}{!}{%
        \begin{tabular}{@{}lccccc@{}}
            \toprule
            \toprule
            \textbf{Objective} & \textbf{$\pi^{(2)}=0.5$} & \textbf{$\pi^{(2)}=0.6$} & \textbf{$\pi^{(2)}=0.7$} & \textbf{$\pi^{(2)}=0.8$} & \textbf{$\pi^{(2)}=0.9$}\\ %
            \toprule
                AUC($y_1$) & 0.53	&0.44&	0.52&	0.56	&0.50\\
                \midrule
                
                 AUC($y_2$) &0.49	&0.46&	0.50&	0.54	&0.45\\
                \midrule
                
                AUC$_{\rm LaA}$  &0.06&	\best{0.04}&	\best{0.08}	&\best{0.07}&	\best{0.07} \\
                
                \midrule
                AUC$_{\rm LoA}^{(1, 1)}$  & \best{0.05}	&0.13&	0.09&	0.11&	0.10\\
                \bottomrule
        \end{tabular}%
        }
    \fi
    \caption{$|\text{AUC}(y_1) - \text{AUC}(y_2)|$ on the synthetic dataset from MLP model training. Each column value $(0.5, \ldots, 0.9)$ corresponds to the marginal probability for the second label, $\mathbf{P}(y_2 = 1)$, while the marginal probability for the first label is fixed, $\mathbf{P}(y_1 = 1)=0.5$. We find label aggregation yields the lowest difference between per-label AUC metrics. Here, AUC$_{\rm LaA}$ = AUC($y_1$ + $y_2$) and AUC$_{\rm LoA}^{(1, 1)}$ = AUC($y_1$) + AUC($y_2$).  
    }
    \label{tbl:synth_train}
\end{table}
We generate instances $x \in \mathbb{R}^2$ from a uniform distribution over $[-1,1]^2$. 
We next draw uniform random vectors $\mu_1$, $\mu_2$, $\mu_3$, $\mu_4$ over $[0,1]^2$
controlling the skewness label distributions $\eta^{(1)}(x)$ and $\eta^{(2)}(x)$ (see Appendix~\ref{s:synth_training_experiments} for details of the data generation process).
Depending on the sampled values, we end up with different marginal skewness of the distributions. 
We fix the marginal $\pi^{(1)}=0.5$ and vary $\pi^{(2)}$ to control the effect label skewness on the performance of different techniques.
We train a multi-layer perceptron model 
with
varying level of skewness for the second label, while keeping the first label balanced.
We report the results in Table~\ref{tbl:synth_train}. 
We find that the difference between the per-label AUC is the smallest for label aggregation when the distribution for the second label is imbalanced.
Thus, label aggregation better balances between the two objectives, 
whereas loss aggregation tends to favor one over the other.

%% file: proof_AUC_LA_asymp.tex
\begin{proof}[Proof of \cref{thm:auc-gap}]
Given a function $f:\XCal\to\mathbb{R}$, we first rewrite $\auc\left(f;\{D^{(k)}\}\right)$ as follows:
\begin{align*}
\auc_{\textnormal{LaA}}\left(f;\{D^{(k)}\}\right)={}& \bE\left[H(f(\mathsf{X}_2)-f(\mathsf{X}_1)\mid\sum_{k\in[K]}a_{k}\mathsf{Y}^{(k)}_1<\sum_{k\in[K]}a_{k}\mathsf{Y}^{(k)}_2\right]\\
={}& \frac{\bE\left[H(f(\mathsf{X}_2)-f(\mathsf{X}_1)\mathds{1}_{\{\sum_{k\in[K]}a_{k}\mathsf{Y}^{(k)}_1<\sum_{k\in[K]}a_{k}\mathsf{Y}^{(k)}_2\}}\right]}{\Pr\left(\sum_{k\in[K]}a_{k}\mathsf{Y}^{(k)}_1<\sum_{k\in[K]}a_{k}\mathsf{Y}^{(k)}_2\right)}\\
={}& \frac{\bE\left[H(f(\mathsf{X}_2)-f(\mathsf{X}_1)\mathds{1}_{\{\sum_{k\in[K]}a_{k}\mathsf{Y}^{(k)}_1<\sum_{k\in[K]}a_{k}\mathsf{Y}^{(k)}_2\}}\right]}{c_{K}},
\end{align*}
where 
\begin{align*}
c_{K} \triangleq{}& \Pr\left(\sum_{k\in[K]}a_{k}\mathsf{Y}^{(k)}_1<\sum_{k\in[K]}a_{k}\mathsf{Y}^{(k)}_2\right)\\
={}& \frac{1}{2}\left(1-\Pr\left(\sum_{k\in[K]}a_{k}\mathsf{Y}^{(k)}_1-\sum_{k\in[K]}a_{k}\mathsf{Y}^{(k)}_2=0\right)\right).
\end{align*}

Define 
\[
\nu(\mathsf{X}_1,\mathsf{X}_2)\triangleq\frac{\sum_{k\in[K]}a_{k}\left(\eta_{k}(\mathsf{X}_2)-\eta_{k}(\mathsf{X}_1)\right)}{\sqrt{\sum_{k\in[K]}a_{k}^{2}\sum_{i\in\{1,2\}}\eta_{k}(\mathsf{X}_{i})\left(1-\eta_{k}(\mathsf{X}_{i})\right)}}
\]
and the functionals
\begin{align*}
F_{K}(f) \triangleq{}& \bE\left[H(f(\mathsf{X}_2)-f(\mathsf{X}_1))\mathds{1}_{\{\sum_{k\in[K]}a_{k}\mathsf{Y}^{(k)}_1<\sum_{k\in[K]}a_{k}\mathsf{Y}^{(k)}_2\}}\right],\\
F'_{K}(f) \triangleq{}& \bE\left[H(f(\mathsf{X}_2)-f(\mathsf{X}_1))\Phi\left(\nu\left(\mathsf{X}_1,\mathsf{X}_2\right)\right)\right],
\end{align*}
where $\Phi(\cdot)$ is the cumulative distribution function of the standard normal distribution. 

Next, we will show that $\tilde{f}$ maximizes $F'_{K}(\cdot)$. In other words, we aim to prove that for any $f$, 
\[
F'_{K}(f)\le F'_{K}(\tilde{f}).
\]

On one hand, we have
\begin{align*}
F'_{K}(f) ={}& \frac{1}{2}\bE\left[H(f(\mathsf{X}_2)-f(\mathsf{X}_1))\Phi\left(\nu\left(\mathsf{X}_1,\mathsf{X}_2\right)\right)+H(f(\mathsf{X}_1)-f(\mathsf{X}_2))\Phi\left(\nu\left(\mathsf{X}_2,\mathsf{X}_1\right)\right)\right]\\
={}& \frac{1}{2}\bE\left[H(f(\mathsf{X}_2)-f(\mathsf{X}_1))\Phi\left(\nu\left(\mathsf{X}_1,\mathsf{X}_2\right)\right)+H(f(\mathsf{X}_1)-f(\mathsf{X}_2))\left(1-\Phi\left(\nu\left(\mathsf{X}_1,\mathsf{X}_2\right)\right)\right)\right]\\
\le{}& \frac{1}{2}\bE\left[\max\left\{ \Phi\left(\nu\left(\mathsf{X}_1,\mathsf{X}_2\right)\right),1-\Phi\left(\nu\left(\mathsf{X}_1,\mathsf{X}_2\right)\right)\right\} \right].
\end{align*}

On the other hand, 
let $S_1 = \sum_{k\in[K]}a_{k}\eta_{k}(\mathsf{X}_1)$, $S_2 = \sum_{k\in[K]}a_{k}\eta_{k}(\mathsf{X}_2)$, and $\Phi_{\nu} = \Phi\left(\nu\left(\mathsf{X}_1,\mathsf{X}_2\right)\right)$.
Since $\phi$ is strictly increasing, $H(\tilde{f}(\mathsf{X}_2) - \tilde{f}(\mathsf{X}_1)) = H(S_2 - S_1)$ and $H(\tilde{f}(\mathsf{X}_1) - \tilde{f}(\mathsf{X}_2)) = H(S_1 - S_2)$. Then,
 \begin{align*}
 F'_{K}(\tilde{f}) ={}& \bE\left[H(\tilde{f}(\mathsf{X}_2)-\tilde{f}(\mathsf{X}_1))\Phi\left(\nu\left(\mathsf{X}_1,\mathsf{X}_2\right)\right)\right]\\
 ={}& \frac{1}{2}\bE\Big[H(\tilde{f}(\mathsf{X}_2)-\tilde{f}(\mathsf{X}_1))\Phi_{\nu}  + H(\tilde{f}(\mathsf{X}_1)-\tilde{f}(\mathsf{X}_2))\left(1-\Phi_{\nu}\right)\Big]\\
 ={}& \frac{1}{2}\bE\Bigg[\left(\mathds{1}_{\left\{ S_2 > S_1 \right\} }+\frac{1}{2}\cdot\mathds{1}_{\left\{ S_2 = S_1 \right\} }\right)\Phi_{\nu} +\left(\mathds{1}_{\left\{ S_1 > S_2 \right\} }+\frac{1}{2}\cdot\mathds{1}_{\left\{ S_2 = S_1 \right\} }\right)\left(1-\Phi_{\nu}\right)\Bigg]\\
 ={}& \frac{1}{2}\bE\left[\max\left\{ \Phi_{\nu} , 1-\Phi_{\nu} \right\} \right].
 \end{align*}

Now, we proceed to bound the gap between $F_{K}(f)$ and $F'_{K}(f)$. We first rewrite $F_{K}(f)$ as
\begin{align*}
& F_{K}(f)\\
={}& \bE\left[H(f(\mathsf{X}_2)-f(\mathsf{X}_1))\bE\left[\mathds{1}_{\{\sum_{k\in[K]}a_{k}\mathsf{Y}^{(k)}_1<\sum_{k\in[K]}a_{k}\mathsf{Y}^{(k)}_2\}}\mid \mathsf{X}_1,\mathsf{X}_2\right]\right]\\
={}& \bE\left[H(f(\mathsf{X}_2)-f(\mathsf{X}_1))\Pr\left(\sum_{k\in[K]}a_{k}\mathsf{Y}^{(k)}_1-\sum_{k\in[K]}\mathsf{Y}^{(k)}_2<0\mid \mathsf{X}_1,\mathsf{X}_2\right)\right]\\
={}& \bE\left[H(f(\mathsf{X}_2)-f(\mathsf{X}_1))\Pr\left(\sum_{k\in[K]}\mathsf{Y}_{k}+\sum_{k\in[K]}\mathsf{Y}'_{k}<\sum_{k\in[K]}a_{k}\left(\eta_{k}(\mathsf{X}_2)-\eta_{k}(\mathsf{X}_1)\right)\mid \mathsf{X}_1,\mathsf{X}_2\right)\right]\\
={}& \bE\left[H(f(\mathsf{X}_2)-f(\mathsf{X}_1))\Pr\left(\mathsf{S}_{K}<\nu\left(\mathsf{X}_1,\mathsf{X}_2\right)\mid \mathsf{X}_1,\mathsf{X}_2\right)\right],
\end{align*}
where 
\begin{align*}
\mathsf{Y}_{k} \triangleq{}& a_{k}\left(\mathsf{Y}^{(k)}_1-\eta_{k}(\mathsf{X}_1)\right),\\
\mathsf{Y}_{k}' \triangleq{}& -a_{k}\left(\mathsf{Y}^{(k)}_2-\eta_{k}(\mathsf{X}_2)\right),\\
\mathsf{S}_{K} \triangleq{}& \frac{\sum_{k\in[K]}\mathsf{Y}_{k}+\sum_{k\in[K]}\mathsf{Y}'_{k}}{\sqrt{\sum_{k\in[K]}a_{k}^{2}\sum_{i\in\{1,2\}}\eta_{k}(\mathsf{X}_{i})\left(1-\eta_{k}(\mathsf{X}_{i})\right)}}.
\end{align*}

Note that $\mathsf{Y}_{k}$ and $\mathsf{Y}'_{k}$ have zero mean, i.e., $\bE\left[\mathsf{Y}_{k}\mid \mathsf{X}_1,\mathsf{X}_2\right]=\bE\left[\mathsf{Y}_{k}'\mid \mathsf{X}_1,\mathsf{X}_2\right]=0$. Moreover, we have 
\begin{align*}
\bE\left[\left|\mathsf{Y}_{k}\right|^{2}\mid \mathsf{X}_1,\mathsf{X}_2\right] ={}& a_{k}^{2}\bE\left[\left|\mathsf{Y}^{(k)}_1-\eta_{k}(\mathsf{X}_1)\right|^{2}\mid \mathsf{X}_1,\mathsf{X}_2\right]\\
={}& a_{k}^{2}\eta_{k}(\mathsf{X}_1)\left(1-\eta_{k}(\mathsf{X}_1)\right),\\
\bE\left[\left|\mathsf{Y}_{k}\right|^{3}\mid \mathsf{X}_1,\mathsf{X}_2\right] ={}& a_{k}^{3}\bE\left[\left|\mathsf{Y}^{(k)}_1-\eta_{k}(\mathsf{X}_1)\right|^{3}\mid \mathsf{X}_1,\mathsf{X}_2\right]\\
={}& a_{k}^{3}\eta_{k}(\mathsf{X}_1)\left(1-\eta_{k}(\mathsf{X}_1)\right)\left[\left(1-\eta_{k}(\mathsf{X}_1)\right)^{2}+\eta_{k}(\mathsf{X}_1)^{2}\right]\\
\le{}& \frac{1}{2}a_{k}^{3}\eta_{k}(\mathsf{X}_1)\left(1-\eta_{k}(\mathsf{X}_1)\right).
\end{align*}
Similarly,
\begin{align*}
\bE\left[\left|\mathsf{Y}_{k}'\right|^{2}\mid \mathsf{X}_1,\mathsf{X}_2\right] ={}& a_{k}^{2}\eta_{k}(\mathsf{X}_2)\left(1-\eta_{k}(\mathsf{X}_2)\right),\\
\bE\left[\left|\mathsf{Y}_{k}'\right|^{3}\mid \mathsf{X}_1,\mathsf{X}_2\right] \le{}& \frac{1}{2}a_{k}^{3}\eta_{k}(\mathsf{X}_2)\left(1-\eta_{k}(\mathsf{X}_2)\right).
\end{align*}

By the Berry-Esseen theorem, for any $t$,
\begin{align*}
&\left|\Pr\left(\mathsf{S}_{K}<t\mid \mathsf{X}_1,\mathsf{X}_2\right)-\Phi\left(t\right)\right| \\
\le{}& \left(\sum_{k\in[K]}\bE\left[\left|\mathsf{Y}_{k}\right|^{2}\mid \mathsf{X}_1,\mathsf{X}_2\right]+\sum_{k\in[K]}\bE\left[\left|\mathsf{Y}_{k}'\right|^{2}\mid \mathsf{X}_1,\mathsf{X}_2\right]\right)^{-3/2}\cdot\\
& \left(\sum_{k\in[K]}\bE\left[\left|\mathsf{Y}_{k}\right|^{3}\mid \mathsf{X}_1,\mathsf{X}_2\right]+\sum_{k\in[K]}\bE\left[\left|\mathsf{Y}_{k}'\right|^{3}\mid \mathsf{X}_1,\mathsf{X}_2\right]\right)\\
\le{}& \frac{\sum_{k\in[K]}a_{k}^{3}\sum_{i\in\{1,2\}}\eta_{k}(\mathsf{X}_{i})\left(1-\eta_{k}(\mathsf{X}_{i})\right)}{2\left(\sum_{k\in[K]}a_{k}^{2}\sum_{i\in\{1,2\}}\eta_{k}(\mathsf{X}_{i})\left(1-\eta_{k}(\mathsf{X}_{i})\right)\right)^{3/2}}.
\end{align*}

We can now bound the gap between $F_{K}(f)$ and $F'_{K}(f)$:
\begin{align*}
&\left|F_{K}(f)-F'_{K}(f)\right| \\
={}& \left|\bE\left[H(f(\mathsf{X}_2)-f(\mathsf{X}_1)\Pr\left(\mathsf{S}_{K}<\nu\left(\mathsf{X}_1,\mathsf{X}_2\right)\mid \mathsf{X}_1,\mathsf{X}_2\right)\right]-\bE\left[H(f(\mathsf{X}_2)-f(\mathsf{X}_1))\Phi\left(\nu\left(\mathsf{X}_1,\mathsf{X}_2\right)\right)\right]\right|\\
\le{}& \bE\left[\left|H(f(\mathsf{X}_2)-f(\mathsf{X}_1)\right|\left|\Pr\left(\mathsf{S}_{K}<\nu\left(\mathsf{X}_1,\mathsf{X}_2\right)\mid \mathsf{X}_1,\mathsf{X}_2\right)-\Phi\left(\nu\left(\mathsf{X}_1,\mathsf{X}_2\right)\right)\right|\right]\\
\le{}& \bE\left[\left|\Pr\left(\mathsf{S}_{K}<\nu\left(\mathsf{X}_1,\mathsf{X}_2\right)\mid \mathsf{X}_1,\mathsf{X}_2\right)-\Phi\left(\nu\left(\mathsf{X}_1,\mathsf{X}_2\right)\right)\right|\right]\\
\le{}& \bE\left[\frac{\sum_{k\in[K]}a_{k}^{3}\sum_{i\in\{1,2\}}\eta_{k}(\mathsf{X}_{i})\left(1-\eta_{k}(\mathsf{X}_{i})\right)}{2\left(\sum_{k\in[K]}a_{k}^{2}\sum_{i\in\{1,2\}}\eta_{k}(\mathsf{X}_{i})\left(1-\eta_{k}(\mathsf{X}_{i})\right)\right)^{3/2}}\right].
\end{align*}

Therefore,
\begin{align*}
& F_{K}(f^{*})-F_{K}(\tilde{f})\\
={}& \left(F_{K}(f^{*})-F'_{K}(f^{*})\right)+\left(F'_{K}(f^{*})-F'_{K}(\tilde{f})\right)+\left(F'_{K}(\tilde{f})-F_{K}(\tilde{f})\right)\\
\le{}& \bE\left[\frac{\sum_{k\in[K]}a_{k}^{3}\sum_{i\in\{1,2\}}\eta_{k}(\mathsf{X}_{i})\left(1-\eta_{k}(\mathsf{X}_{i})\right)}{\left(\sum_{k\in[K]}a_{k}^{2}\sum_{i\in\{1,2\}}\eta_{k}(\mathsf{X}_{i})\left(1-\eta_{k}(\mathsf{X}_{i})\right)\right)^{3/2}}\right].
\end{align*}

Next, we establish an upper bound for $\Pr\left(\sum_{k\in[K]}a_{k}\mathsf{Y}^{(k)}_1-\sum_{k\in[K]}a_{k}\mathsf{Y}^{(k)}_2=0\right)$:
\begin{align*}
&\Pr\left(\sum_{k\in[K]}a_{k}\mathsf{Y}^{(k)}_1-\sum_{k\in[K]}a_{k}\mathsf{Y}^{(k)}_2=0\right)\\
={}& \bE\left[\Pr\left(\sum_{k\in[K]}a_{k}\mathsf{Y}^{(k)}_1-\sum_{k\in[K]}a_{k}\mathsf{Y}^{(k)}_2=0\mid \mathsf{X}_1,\mathsf{X}_2\right)\right]\\
={}& \bE\left[\Pr\left(\mathsf{S}_{K}=\nu\left(\mathsf{X}_1,\mathsf{X}_2\right)\mid \mathsf{X}_1,\mathsf{X}_2\right)\right]\\
={}& \bE\left[\Pr\left(\mathsf{S}_{K}\le\nu\left(\mathsf{X}_1,\mathsf{X}_2\right)\mid \mathsf{X}_1,\mathsf{X}_2\right)-\lim_{t\to\nu(\mathsf{X}_1,\mathsf{X}_2)^{-}}\Pr\left(\mathsf{S}_{K}\le t\mid \mathsf{X}_1,\mathsf{X}_2\right)\right]\\
={}& \bE\left[\Pr\left(\mathsf{S}_{K}\le\nu\left(\mathsf{X}_1,\mathsf{X}_2\right)\mid \mathsf{X}_1,\mathsf{X}_2\right)-\Phi\left(\nu\left(\mathsf{X}_1,\mathsf{X}_2\right)\right)-\lim_{t\to\nu(\mathsf{X}_1,\mathsf{X}_2)^{-}}\left(\Pr\left(\mathsf{S}_{K}\le t\mid \mathsf{X}_1,\mathsf{X}_2\right)-\Phi\left(t\right)\right)\right].
\end{align*}

Taking the absolute value, we get
\begin{align*}
&\Pr\left(\sum_{k\in[K]}a_{k}\mathsf{Y}^{(k)}_1-\sum_{k\in[K]}a_{k}\mathsf{Y}^{(k)}_2=0\right)\\
\le{}& \bE\left[\left|\Pr\left(\mathsf{S}_{K}\le\nu\left(\mathsf{X}_1,\mathsf{X}_2\right)\mid \mathsf{X}_1,\mathsf{X}_2\right)-\Phi\left(\nu\left(\mathsf{X}_1,\mathsf{X}_2\right)\right)\right|+\lim_{t\to\nu(\mathsf{X}_1,\mathsf{X}_2)^{-}}\left|\left(\Pr\left(\mathsf{S}_{K}\le t\mid \mathsf{X}_1,\mathsf{X}_2\right)-\Phi\left(t\right)\right)\right|\right]\\
\le{}& \bE\left[\frac{\sum_{k\in[K]}a_{k}^{3}\sum_{i\in\{1,2\}}\eta_{k}(\mathsf{X}_{i})\left(1-\eta_{k}(\mathsf{X}_{i})\right)}{\left(\sum_{k\in[K]}a_{k}^{2}\sum_{i\in\{1,2\}}\eta_{k}(\mathsf{X}_{i})\left(1-\eta_{k}(\mathsf{X}_{i})\right)\right)^{3/2}}\right].
\end{align*}

Consequently, we obtain a lower bound for $c_{K}$:
\[
c_{K}\ge\max\left\{ \frac{1}{2}\left(1-\bE\left[\frac{\sum_{k\in[K]}a_{k}^{3}\sum_{i\in\{1,2\}}\eta_{k}(\mathsf{X}_{i})\left(1-\eta_{k}(\mathsf{X}_{i})\right)}{\left(\sum_{k\in[K]}a_{k}^{2}\sum_{i\in\{1,2\}}\eta_{k}(\mathsf{X}_{i})\left(1-\eta_{k}(\mathsf{X}_{i})\right)\right)^{3/2}}\right]\right),0\right\}.
\]

Recall that $\auc_\textnormal{LaA}(f;\{D^{(k)}\})=\frac{F_{K}(f)}{c_{K}}$. We then have
\begin{align*}
&\operatorname{AUC}_\textnormal{LaA}\left(f^{*};\{D^{(k)}\}\right)-\operatorname{AUC}_\textnormal{LaA}\left(\tilde{f};\{D^{(k)}\}\right)\\
={}& \frac{F_{K}(f^{*})-F_{K}(\tilde{f})}{c_{K}}\\
\le{}& \frac{2 \cdot \bE\left[\frac{\sum_{k\in[K]}a_{k}^{3}\sum_{i\in\{1,2\}}\eta_{k}(\mathsf{X}_{i})\left(1-\eta_{k}(\mathsf{X}_{i})\right)}{\left(\sum_{k\in[K]}a_{k}^{2}\sum_{i\in\{1,2\}}\eta_{k}(\mathsf{X}_{i})\left(1-\eta_{k}(\mathsf{X}_{i})\right)\right)^{3/2}}\right]}{\left(1-\bE\left[\frac{\sum_{k\in[K]}a_{k}^{3}\sum_{i\in\{1,2\}}\eta_{k}(\mathsf{X}_{i})\left(1-\eta_{k}(\mathsf{X}_{i})\right)}{\left(\sum_{k\in[K]}a_{k}^{2}\sum_{i\in\{1,2\}}\eta_{k}(\mathsf{X}_{i})\left(1-\eta_{k}(\mathsf{X}_{i})\right)\right)^{3/2}}\right]\right)}\\
={}& \psi\left(\bE\left[\frac{\sum_{k\in[K]}a_{k}^{3}\sum_{i\in\{1,2\}}\eta_{k}(\mathsf{X}_{i})\left(1-\eta_{k}(\mathsf{X}_{i})\right)}{\left(\sum_{k\in[K]}a_{k}^{2}\sum_{i\in\{1,2\}}\eta_{k}(\mathsf{X}_{i})\left(1-\eta_{k}(\mathsf{X}_{i})\right)\right)^{3/2}}\right]\right).
\end{align*}
\end{proof}

\begin{proof}[Proof of \cref{cor:auc-gap}]
By \cref{thm:auc-gap}, we have
\begin{align*}
& \operatorname{AUC}_\textnormal{LaA}\left(f^{*};\{D^{(k)}\}\right)-\operatorname{AUC}_\textnormal{LaA}\left(\tilde{f};\{D^{(k)}\}\right) \\
\le{}& \psi\left(\bE\left[\frac{1}{\left(\sum_{k\in[K]}\sum_{i\in\{1,2\}}\eta_{k}(\mathsf{X}_{i})\left(1-\eta_{k}(\mathsf{X}_{i})\right)\right)^{1/2}}\right]\right)\\
\le{}& \psi\left(\frac{1}{\sqrt{2c(1-c)K}}\right)\\
={}& O\left(\frac{1}{\sqrt{K}}\right).
\end{align*}
\end{proof}